\documentclass{article}

\usepackage{microtype}
\usepackage{graphicx}
\usepackage{booktabs} 
\usepackage{balance}
\usepackage{hyperref}
\usepackage{enumitem}

\usepackage[accepted]{icml2025} 

\usepackage{amsmath}
\usepackage{amssymb}
\usepackage{mathtools}
\usepackage{amsthm}

\usepackage[capitalize,noabbrev]{cleveref}

\theoremstyle{plain}
\newtheorem{theorem}{Theorem}[section]

\theoremstyle{definition}
\newtheorem{definition}[theorem]{Definition}

\theoremstyle{remark}
\newtheorem{remark}[theorem]{Remark}

\usepackage[textsize=tiny]{todonotes}

\usepackage{times}
\usepackage{helvet}
\usepackage{courier}
\usepackage{pgfplots}
\usepackage{graphicx}
\usepgfplotslibrary{groupplots}
\pgfplotsset{compat=newest}
\usepackage{comment}
\usepackage{amsmath}
\usepackage{amsthm}
\usepackage{amsfonts}
\usepackage{centernot}
\usepackage{tikz}

\usepackage{amsmath,amssymb,amsfonts}
\usepackage{algorithmic}
\usepackage{algorithm} 
\usepackage{graphicx}
\usepackage{natbib}
\usepackage{booktabs}
\usepackage{siunitx}
\usepackage{thmtools} 
\usepackage{thm-restate}
\usepackage{lipsum}
\usepackage{multirow}
\usepackage{mathtools}
\usepackage{graphicx}
\usepackage{booktabs}
\usepackage{adjustbox}
\usepackage{array}
\usepackage{pgfplots}
\pgfplotsset{compat=newest}
\usepackage{float}
\usepackage{soul}
\usepackage{xcolor}

\newcommand{\q}[1]{\textcolor{purple}{#1}}
\usepackage{icml2025}
\newcommand{\dn}{q}

\usepackage{pgfplots}
\usepackage{caption}
\usepackage{subcaption}

\usepackage{amssymb}
\usepackage{pifont}
\newcommand{\cmark}{\ding{51}}%
\newcommand{\xmark}{\ding{55}}%

\definecolor{darkblue}{rgb}{0, 0.2, 0.4}

\newcommand{\sd}[1]{{{\footnotesize±}{\scriptsize#1}}}
\usepackage{url}
\usepackage{hyperref} 
 \hypersetup{ 
     colorlinks=true, 
     linkcolor=black, 
     filecolor=black, 
     citecolor = black,       
     urlcolor=blue, 
     pdfborderstyle={/S/U/W 1}
     } 

\makeatletter
\let\ps@oldplain\ps@plain
\def\ps@plain{%
  \ps@oldplain
  \def\@oddfoot{ Correspondence to: Asela Hevapathige $<$asela.hevapathige@anu.edu.au$>$.\hfil}%
  \def\@evenfoot{This is a left-aligned footer\hfil}%
}
\makeatother

\icmltitlerunning{How Graph Neural Networks Capture Structural Interactions?}

\begin{document}

\twocolumn[
\icmltitle{Permutation-Invariant Graph Partitioning: \\How Graph Neural Networks Capture Structural Interactions?}




\begin{icmlauthorlist}
\icmlauthor{Asela Hevapathige}{yyyy}
\icmlauthor{Qing Wang}{yyyy}
\\
\icmlauthor{$^1$ Graph Research Lab, School of Computing, Australian National University}{}
\end{icmlauthorlist}



\icmlkeywords{Machine Learning, ICML}

\vskip 0.3in
]




\begin{abstract}
Graph Neural Networks (GNNs) have paved the way for being a cornerstone in graph-related learning tasks. Yet, the ability of GNNs to capture structural interactions within graphs remains under-explored. In this work, we address this gap by drawing on the insight that permutation invariant graph partitioning enables a powerful way of exploring structural interactions. We establish theoretical connections between permutation invariant graph partitioning and graph isomorphism, and then propose Graph Partitioning Neural Networks (GPNNs), a novel architecture that efficiently enhances the expressive power of GNNs in learning structural interactions. We analyze how partitioning schemes and structural interactions contribute to GNN expressivity and their trade-offs with complexity. Empirically, we demonstrate that GPNNs outperform existing GNN models in capturing structural interactions across diverse graph benchmark tasks. 
\end{abstract}

\section{Introduction}

Graph Neural Networks (GNNs) have emerged as the \emph{de facto} standard for tackling graph learning tasks \citep{horn2021:topological, bodnar2021:weisfeiler, bouritsas2022:improving}. Among the rich landscape of GNN models, Message-Passing Neural Networks (MPNNs) stand out for their simplicity and efficiency, making them a popular choice for real-world applications. This is due to their ability to leverage local neighborhood information through a message-passing scheme, which iteratively computes node representations~\citep{gilmer2017:neural,kipf2016:semi}.

However, the representational power of MPNNs is fundamentally constrained by the Weisfeiler-Lehman test (1-WL) \citep{weisfeiler1968reduction,xu2018:powerful,morris2019:weisfeiler}. To overcome this limitation, researchers have proposed various approaches to enhance MPNN expressivity beyond 1-WL, including injecting structural properties \citep{bouritsas2022:improving,barcelo2021:graph,wijesinghe2022:new,xu2024union}, incorporating higher-order substructures \citep{morris2019:weisfeiler,morris2020:weisfeiler,abu2019:mixhop}, expanding receptive fields for message passing \citep{nikolentzos2020:k,feng2022:powerful,wang2022:towards}, leveraging subgraphs \citep{zhao2021:stars,zhang2021:nested,wang2022:glass,bevilacqua2021:equivariant,cotta2021:reconstruction}, integrating homomorphism counts \cite{zhangbeyond,jinhomomorphism} and applying inductive coloring techniques \citep{you2021:identity,huang2022:boosting}.
Despite these advancements, there is still limited understanding of how structural components within a graph, such as subgraphs representing diverse properties, interact and influence expressivity in learning tasks. Existing methods primarily focus on straightforward interactions between nodes and their neighbors, leaving the intricate interplay among structural components underexplored. Addressing these gaps is crucial for unlocking the full potential of expressive GNNs.

\begin{figure}[t]
  \centering
  \includegraphics[width=0.48\textwidth]{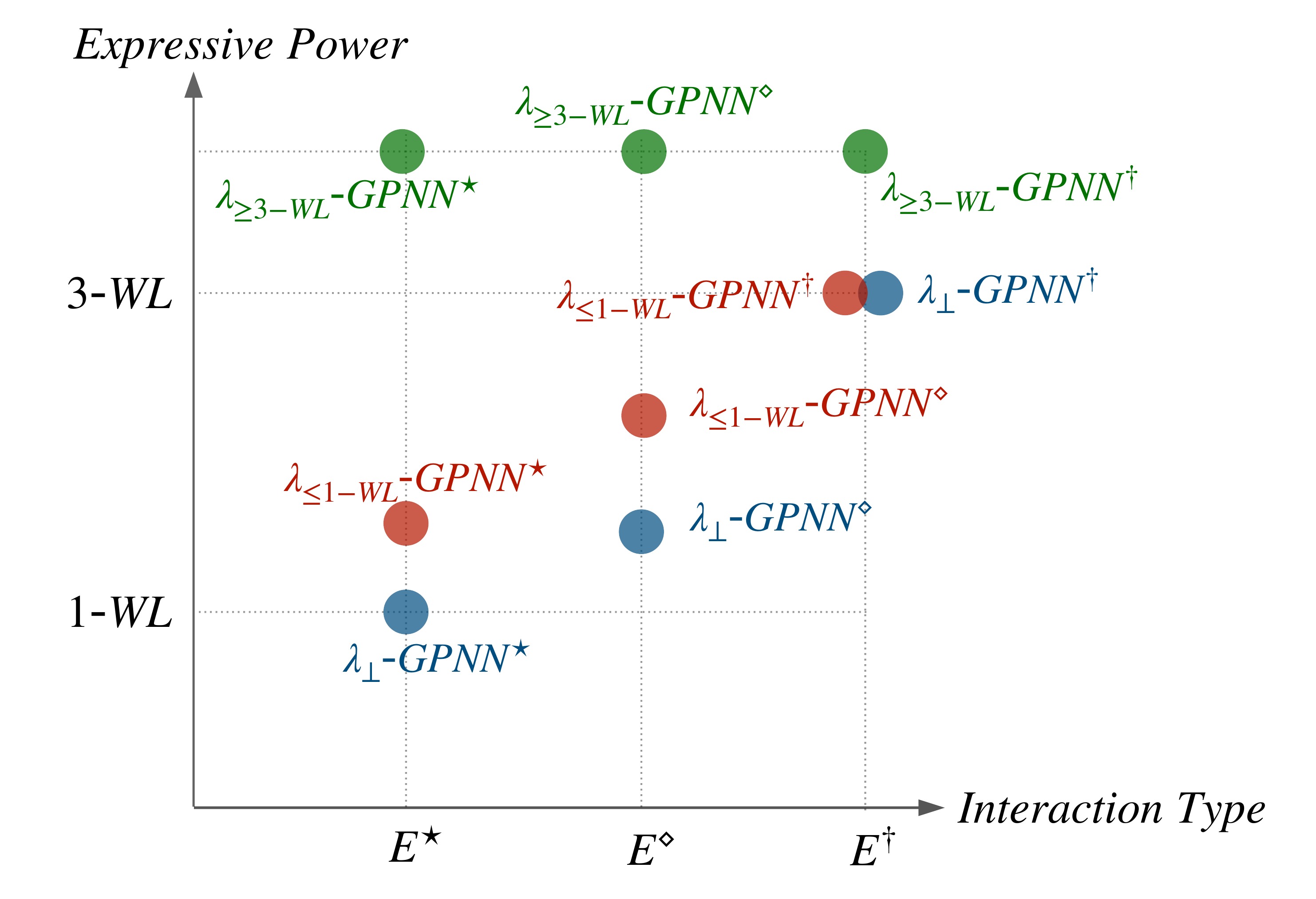}
  \caption{The expressivity of GPNN variants is analysed in terms of (1) \textit{interaction type}: $E^{\star}$ (intra-edges), $E^{\diamond}$ (intra-edges and inter-edges), and $E^{\dagger}$ (all edges); (2) \textit{expressive power}: $1$-WL and $3$-WL, under different permutation-invariant graph partitioning schemes $\lambda_{\bot}$, $\lambda_{\leq1\text{-}WL}$, and $\lambda_{\geq3\text{-}WL}$. Detailed theoretical results are provided in the section ``Expressivity Analysis".}
  \label{fig:hierarchy}
  \end{figure}

\medskip
\noindent\textbf{Present work.~} We aim to address the aforementioned limitations by observing that partitioning a graph into subgraphs that preserve structural properties provides a powerful means to exploit interactions among different structural components of the graph. Notably, graphs in real-world applications often consist of various structural components that can be distinguished by their topological properties. By exploring these structural components and their interactions, we can gain novel insights into the structure of a graph and enhance the power of graph representations.

Nevertheless, unlike other data types such as images or sequences, graphs lack an inherent, consistent ordering of vertices. Therefore, a model that learns representations of graphs must treat isomorphic graphs -- those that differ only by vertex permutations -- as identical. To achieve this, ensuring permutation invariance is crucial during graph partitioning. This guarantees that graphs are divided based on their intrinsic structural properties, rather than being influenced by arbitrary vertex orderings.


Inspired by these insights, we explore how permutation-invariant graph partitioning enhances GNN expressivity efficiently. Our key contributions are:
\begin{itemize}[leftmargin=10pt,itemsep=-0cm] 
  \item \textbf{Theoretical insights:} We establish a novel connection between graph partitioning and graph isomorphism by introducing the concepts of \emph{partition isomorphism} and \emph{interaction isomorphism}, providing a rigorous theoretical framework for analyzing structural interactions. 
  \item \textbf{GNN architecture:}~We propose \emph{Graph Partitioning Neural Networks} (GPNNs), a new GNN architecture that integrates structural interactions through permutation-invariant graph partitioning to enhance graph representation learning.\looseness=-1
  \item \textbf{Expressivity analysis:}~We theoretically prove that GPNNs surpass the expressive power of 1-WL and efficiently approach the expressive power of 3-WL, addressing critical limitations of traditional GNNs.
  \item \textbf{Complexity analysis:}~We show that GPNN variants achieve a balance between computational efficiency and representational power, making them practical for large-scale graph learning tasks.
\item \textbf{Practical partitioning:}~We explore efficient, permutation -invariant graph partitioning schemes, such as 
$k$-core and degree-based partitioning, to enable meaningful subgraph extraction and preserve structural properties. \looseness=-1
\end{itemize}
Figure~\ref{fig:hierarchy} illustrates the theoretical results on the expressivity of the proposed GPNN architecture in relation to $k$-WL~\cite{cai1992:optimal} and interaction types across different graph partitioning schemes. To empirically verify the theoretical designs of GPNN, we conduct experiments on diverse graph benchmark tasks, demonstrating its superior performance over state-of-the-art models. 

\begin{figure*}\vspace*{-0.5cm}
\begin{minipage}{0.03\linewidth}
\hspace*{0.2cm}
\end{minipage}\begin{minipage}{0.1\linewidth}
\vspace{0.4cm}
Example graph pairs

\vspace{0.5cm}
Boundary sugraphs 

\vspace{0.5cm}
Partitioned subgraphs

\end{minipage}
\begin{minipage}{0.6\linewidth}
    \centering\vspace{0.4cm}\includegraphics[width=0.87\textwidth]{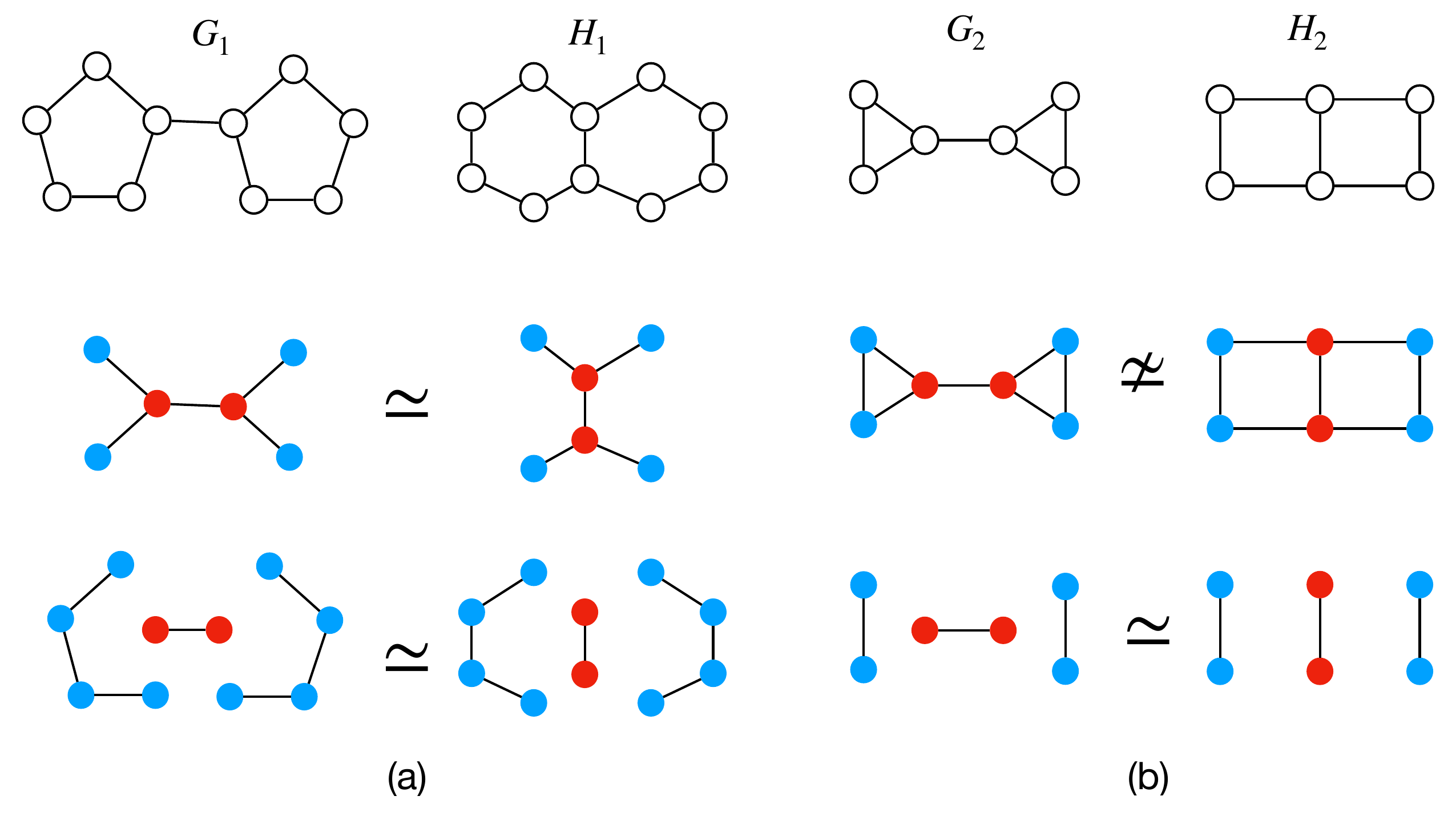}
\end{minipage}
\begin{minipage}{0.14\linewidth}
\vspace*{0cm}

\hspace*{0.1cm}\scalebox{0.9}{\begin{tabular}{c|ccc} 
 \toprule
 \multirow{2}{*}{\hspace{-2mm}Type} &\multicolumn{2}{|c}{Pairs of Graphs}\\\cline{2-3}
  &\hspace{2mm}(a)&(b)\\\midrule
Graph  & \multirow{2}{*}{\xmark} & \multirow{2}{*}{\xmark}  \\ 
Isomorphism  &&\\
\\
Interaction  &\multirow{2}{*}{\cmark} & \multirow{2}{*}{\xmark}\\ 
Isomorphism &&\\
\\
Partition & \multirow{2}{*}{\cmark} & \multirow{2}{*}{\cmark}\\
Isomorphism &&\\
 \bottomrule
\end{tabular}}
\end{minipage}\caption{Two pairs of non-isomorphic graphs, $(G_1, H_1)$ and $(G_2, H_2)$, are partitioned by node degrees, grouping nodes with the same degree. Boundary and partitioned subgraphs are shown, with different subgraphs highlighted in red and blue. (a) $G_1 \stackrel{PI}{\simeq} H_1$ and $G_1 \stackrel{II}{\simeq} H_1$; (b) $G_2 \stackrel{PI}{\simeq} H_2$, but $G_2 \not\stackrel{II}{\simeq} H_2$.}\label{fig:two-isomorphism}
\end{figure*}

\section{Related Work} 

\paragraph{ GNNs beyond 1-WL.} The Weisfeiler-Lehman algorithm \cite{weisfeiler1968:reduction} has commonly been used as a yardstick for measuring the expressive power of GNNs. Since \citet{xu2018:powerful,morris2019:weisfeiler} showed that the expressive power of standard GNNs is upper-bounded by 1-WL, numerous studies have explored the connections between the $k$-WL hierarchy \cite{cai1992:optimal} and the expressive power of GNNs \cite{zhao2022:practical}, which has led to the development of higher-order GNNs~\cite{maron2019:universality,morris2019:weisfeiler,morris2020:weisfeiler,zhao2022:practical}. Apart from higher-order GNNs, there are a variety of GNN models proposed in the literature, which enhance the expressive power of MPNNs to go beyond 1-WL. Some of these models extract structural information using a pre-processing step and
inject structural information into a message-passing scheme as node or edge features~\cite{horn2021:topological,barcelo2021:graph,bouritsas2022:improving,wijesinghe2022:new}. Some models apply base GNNs such as GIN~\cite{xu2018:powerful} at a subgraph level rather than the entire graph based on the observation that 1-WL indistinguishable graphs may  contain 1-WL distinguishable subgraphs \cite{bevilacqua2021:equivariant,cotta2021:reconstruction,zeng2023substructure,frasca2022understanding}. We refer the reader to the survey articles by \citet{sato2020:survey} and \citet{zhang2024expressive} for a detailed discussion on expressive GNNs. 

 
\paragraph{Partitioning-based GNNs.}Graph partitioning has previously been explored in the domain of GNNs. Several works have employed partitioning for speeding up large-scale graph processing~\cite{mu2023:graph,t:46656,miao2021:degnn,chiang2019:cluster,lin2021large,yang2023betty,wan2023scalable}, graph similarity computation \cite{xu2021:graph,xu2020hierarchical}, and graph matching \cite{yeefficient,he2021learnable}. The main idea of these works is to use partitioning techniques to divide graphs into smaller components, which reduces memory and computational requirements during model training. This approach also enables parallel processing and reduces model parameter space, leading to faster training and potentially better representations for large-scale graphs. \looseness=-1

Our work fundamentally differs from existing partitioning based GNNs by focusing on representation learning rather than computational efficiency. We employ permutation-invariant graph partitions to uncover structural interactions within and between components, encoding these intersections into graph representations. To our knowledge, no prior studies have explored how graph partitioning can be integrated into representation learning to capture structural interactions within graphs. Given that permutation invariance is a fundamental inductive bias in graph learning, a key challenge lies in incorporating graph partitioning into representation learning while preserving the permutation-invariance properties of graphs. This work is the first to tackle this challenge and to explore the connection between permutation-invariant graph partitioning and graph isomorphism to enhance the expressive power of GNNs.

\section{Graph Partitioning Scheme}

In this section, we define permutation-invariant graph partitioning and explain its connection to graph isomorphism. 

Let $G$ be a simple graph, and $V(G)$ and $E(G)$ refer to the vertex set and the edge set of $G$, respectively. 
Given $d\in \mathbb{N}$, for each $v\in V(G)$, the set of $d$-hop neighbouring vertices is denoted as $N_d(v) = \{u \in V(G) \mid (u,v)\in E(G), \delta(u,v)\leq d\}$, where $\delta(u,v)$ refers to the shortest-path distance between two vertices $u$ and $v$. When $d$ is equal to 1, we denote it as $N(v)$.
An induced subgraph of $G$ by $U\subseteq V(G)$ is the graph $G[U]$ with a vertex set 
$U$ and edges only between vertices in $U$, denote as $G[U]\subseteq G$.

Two graphs $G_1$ and $G_2$ are \emph{isomorphic}, denoted as $G_1\simeq G_2$, if there exists a bijection $g: V(G_1)\rightarrow V(G_2)$ such that $(v,u)\in E(G_1)$ if and only if $(g(v),g(u))\in E(G_2)$. A permutation $\pi$ on a graph $G$ is a bijection $\pi:V(G)\to V(G)$ such that $\pi(G)=(V(G), \pi(E(G)))$, where $\pi(E(G))=\{(\pi(v),\pi(u)|(v,u)\in E(G)\}$. A function $f$ is \emph{permutation-invariant} if and only if $f(G) = f(\pi(G))$ for any graph $G$ and any permutation $\pi$ on $G$. 
 Let $\mathcal{G}$ be a set of graphs closed under isomorphism. A \emph{subgraph invariant} is a function  $\phi:\mathcal{G}\rightarrow \mathcal{G}$ that determines a subgraph $\phi(G)\subseteq G$ for each $G\in \mathcal{G}$,  preserved under isomorphisms. 

\begin{definition}[Graph Partitioning Scheme] 
Let $P(\mathcal{G})$ be the powerset of $\mathcal{G}$ and $\Phi=\{\phi_j\}_{j\in[1,k]}$ be a family of subgraph invariants. A \emph{graph partitioning scheme} is a permutation invariant function $f_{\Phi}: \mathcal{G}\rightarrow P(\mathcal{G})$ which partitions any graph $G\in \mathcal{G}$ into a set of subgraphs $\{S_1,\dots,S_k\}$, where $\bigcup_{1 \leq i \leq k} V(S_{i}) = V(G)$, $\bigwedge_{1\leq i\neq j\leq k}V(S_i)\cap V(S_j)=\emptyset$, and $\phi_j(G)=S_j$ for $j\in[1,k]$. Note that $S_j$ can be an empty subgraph, i.e., $V(S_j)=\emptyset$ and $E(S_j)=\emptyset$.  
\end{definition}

The permutation-invariant property of a graph partitioning scheme is crucial for exploring structural interactions. This property not only ensures that interactions among different structural components are consistently captured, under any vertex reordering,  but also enables meaningful comparisons between subgraphs with the same indices, partitioned from different graphs by the same graph partition scheme through a family of subgraph invariants.

Based on graph partitioning schemes, we define two notions of isomorphism on graphs: one characterizes interactions within partitioned subgraphs and the other characterizes interactions across partitioned subgraphs.  
 
\begin{definition}[Partition-Isomorphism]\label{def:2}
Two graphs $G$ and $G'$ are \emph{partition-isomorphic}, denoted as $G\stackrel{PI}{\simeq}G'$, with respect to a partitioning scheme $f_{\Phi}$ if there exists a bijective function $g:f_{\Phi}(G)\rightarrow f_{\Phi}(G')$ such that $f_{\Phi}(G)=\{S_1,\dots, S_k\}$, $f_{\Phi}(G')=\{S'_1,\dots, S'_k\}$, and $g(S_i)\simeq S'_i$ for $i=1,\dots, k$.
\end{definition}

 A vertex $v\in V(G)$ is a \emph{border vertex} with respect to $f_{\Phi}$ if there exist two subgraphs $\{S_i,S_j\}\subseteq f_{\Phi}(G)$ such that $(v,u)\in E(G)$, $v\in V(S_i)$, $u\in V(S_j)$ and $S_i\neq S_j$.  
We use $V_B(G)$ to denote the set of all border vertices in $G$. 

\begin{definition}[Boundary Subgraph]
The \emph{boundary subgraph} \( B(G) = (V, E) \) of a graph \( G \) with respect to \( f_{\Phi} \) is defined by the vertex set \( V = V_B(G) \) and the edge set \( E = \{(v, u) \in E(G) \mid v, u \in V_B(G)\} \), where each edge in \( E \) connects two border vertices.  
\end{definition}

Boundary subgraphs underpin structural interaction equivalence in partitioned graphs.

\begin{definition}[Interaction-Isomorphism]\label{def:1}
Two graphs \( G \) and \( G' \) are \emph{interaction-isomorphic}, denoted \( G \stackrel{II}{\simeq} G' \), with respect to a partitioning scheme \( f_\Phi \), if and only if:
\vspace*{-0.3cm}
\begin{itemize}[leftmargin=10pt,itemsep=-0cm]
    \item \( G \stackrel{PI}{\simeq} G' \), i.e., graphs are partition-isomorphic under \( f_\Phi \);
    \item \( B(G) \!\simeq\! B(G') \), i.e., boundary subgraphs are isomorphic.\looseness=-1
\end{itemize}
\end{definition}


The following establishes the connection between partition-isomorphism, interaction-isomorphism, and graph isomorphism. One direction of the proof follows directly from the definitions, while the other is shown using counterexample graphs in Figure \ref{fig:two-isomorphism}. A detailed proof is in the Appendix.\looseness=-1
\begin{restatable}[]{theorem}{thmisomorphism}
\label{theorem-1} Fix a graph partitioning scheme. We have: (a) If $G \simeq G^{'}$, then $G\stackrel{II}{\simeq}G^{'}$, but not vice versa; (b) If $G\stackrel{II}{\simeq}G^{'}$, then $G\stackrel{PI}{\simeq}G^{'}$, but not vice versa.
\end{restatable} 


\begin{figure*}[t!]
  \centering
  \resizebox{1.0\textwidth}{!}{\includegraphics{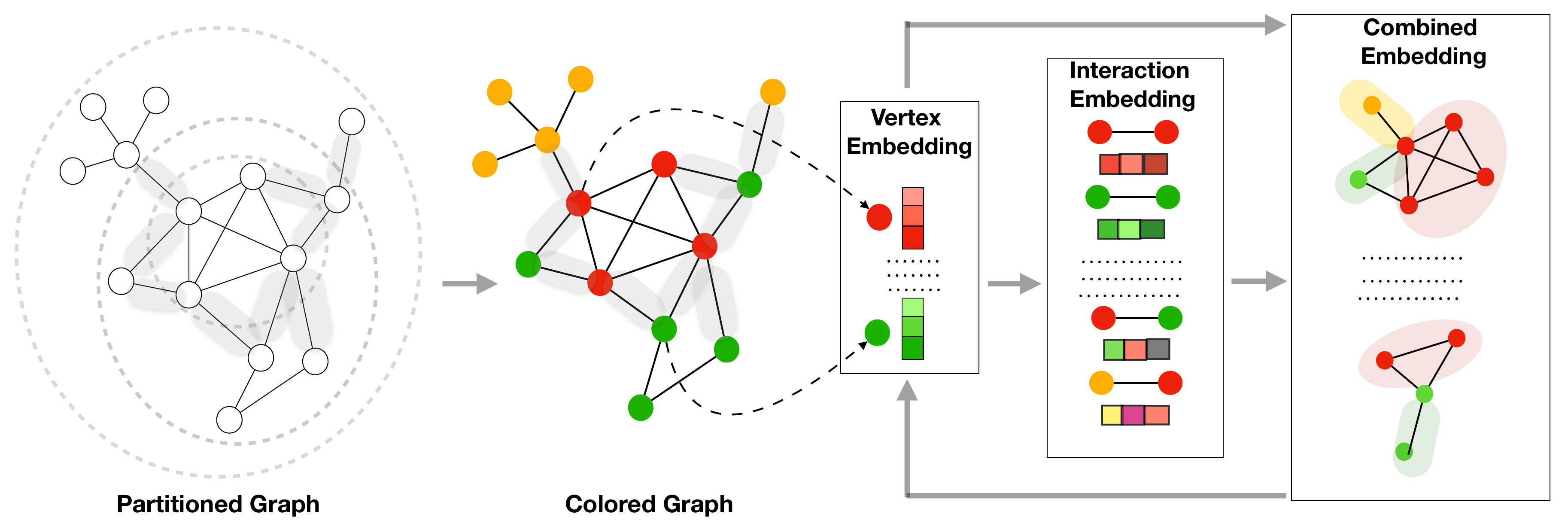}}
  \caption{A high-level workflow of $\lambda_\text{core}$-GPNN for an input graph. The intra-edges are marked using grey color. GPNN generates vertex representations by considering interactions within and between partitions.} 
  \label{fig: partition diagram}\vspace{-0cm}
\end{figure*}


\section{Proposed GNN Architecture} \label{sec:model-architecture}

In this section, we introduce \emph{Graph Partitioning Neural Networks} (GPNN), a novel GNN model which can integrate structural interactions into representation learning.

We begin by defining partitioning colouring, where vertices and edges are assigned colours to reflect permutation-invariant partitions generated by a graph partition scheme.

\begin{definition}[Partition Colouring]
   Let $C_V$ be a set of vertex colours, $C_E$ be a set of edge colours containing $\{c_e, c_a, c_n\}$, $C_V\cap C_E=\emptyset$, $\{S_1,\dots, S_k\}$ be a set of subgraphs generated by a graph partitioning scheme over $G$, and $\pi(\cdot)$ be an injective hashing function. A \emph{partition colouring} $\lambda=(\lambda_V, \lambda_E)$ consists of the following: 
   
   -- \emph{Vertex colouring:} Each vertex $v\in V(G)$ is assigned a \emph{vertex colour}, $\lambda_V: V(G)\rightarrow C_V$ such that $\lambda_V(v)=\lambda_V(u)$ if and only if $v\in V(S_{i})$ and $u\in V(S_{i})$.
   
-- \emph{Edge colouring:} Each pair of distinct vertices $\{v,u\}\subseteq V(G)$ is assigned an \emph{edge colour}, $\lambda_E: V(G)\times V(G)\rightarrow C_E$ such that 
\begin{equation*}
\hspace*{-0.5cm} 
\begin{array}{l}
  \lambda_E(v,u) =\\
    \begin{cases}     
    \pi(\lambda_V(v), c_{e}, \lambda_V(u)) & (v,u) \in E(G), \lambda_V(v)=\lambda_V(u); \\[0.5ex]
    \pi(\lambda_V(v), c_{a}, \lambda_V(u)) & (v,u) \in E(G), \lambda_V(v)\neq\lambda_V(u); \\
    \pi(\lambda_V(v), c_{n},\lambda_V(u)) & (v,u) \notin E(G).
    \end{cases}
\end{array}
\end{equation*}
\end{definition}
  
Here, $\{c_e, c_a, c_n\}$ represents three types of interactions: \emph{inter-interaction}, \emph{intra-interaction}, and \emph{non-interaction}, respectively. Intuitively, inter-interactions refer to interactions between vertices within the same partitioned subgraph, intra-interactions refer to interactions between vertices across different partitioned subgraphs, and non-interactions indicate that no direct interactions occurs between vertices.


\begin{definition}[Coloured Neighbourhood]
   For each vertex $v\in V(G)$, the neighbourhood $N_d(v)$ consists of a set of neighbouring vertex subsets $\{N^1_d(v),\dots, N^k_d(v)\}$, each of which is coloured with a distinct colour under $\lambda_V$ such that, for any two vertices $\{u,w\}\subseteq N_d(v)$, $\lambda_V(u) = \lambda_V(w)$ if and only if $\{u,w\}\subseteq N^j_d(v)$ for exactly one $j\in [1,k]$.
\end{definition}

Below we present the message-passing scheme of GPNN. Given a partition colouring $\lambda=(\lambda_V, \lambda_E)$, let $\beta_v^{(0)}=\lambda_V(v)$ and $\gamma_v^{(0)}=\lambda_V(v)$ for 
any $v\in V(G)$ and $\alpha^{(0)}_{vu}=\lambda_E(v,u)$ for any $v,u\in V(G)$. We learn the vertex embedding $\beta_v^{(\ell+1)}$ of each vertex $v$ at the $(\ell+1)$-th iteration as
\begin{align}\label{eq:node}
\beta_{v}^{(l+1)}= \textsc{Upd}\left( \gamma_{v}^{(l)}, \textsc{Agg}\{\!\!\{\gamma_{u}^{(l)} | u \in N(v) \}\!\!\}\right).
\end{align}
$\textsc{Upd}(.)$ and $\textsc{Agg}(.)$ are injective and permutation-invariant functions that represent update and aggregate operations in GNNs, respectively. Then, we define the interaction embedding $\alpha^{(\ell+1)}_{vu}$ of $(v,u)$ at the $(\ell+1)$-th iteration as  
 \begin{align}\label{eq:edge}
\alpha^{(\ell+1)}_{vu} = \textsc{Upd}\Bigl(\alpha_{vu}^{(\ell)}, \textsc{Agg}\Bigl(\{\!\!\{(\alpha_{vw}^{(\ell)}, \alpha_{uw}^{(\ell)}) | w \in N_d(v)\}\!\!\}\Bigr)\Bigr).
\end{align}
We then aggregate the embeddings of neighboring vertices and their corresponding interaction embeddings from different partitions (i.e., the coloured neigbourhood $\{N^1_d(v),\dots, N^k_d(v)\}$) into a combined embedding $\gamma_v$ as 
\begin{align}
    \gamma_v^{(\ell+1)} &= \textsc{Cmb}\Bigl( \textsc{Agg}\left( \{\!\!\{ (\beta_u^{(\ell+1)}, \alpha_{vu}^{(\ell+1)}) | u \in N_d^1(v) \}\!\!\} \right), \nonumber \\
    &\text{...}, \textsc{Agg}\left( \{\!\!\{ (\beta_u^{(\ell+1)}, \alpha_{vu}^{(\ell+1)}) | u \in N_d^k(v) \}\!\!\} \right) \Bigr).
    \label{eq:combine}
\end{align}

$\textsc{Cmb}(.)$ is an injective combining function, closed under permutation. Note that $(v,u)$ of $\alpha^{(\ell+1)}_{vu}$ may correspond to inter-interaction, intra-interaction, or non-interaction, depending on not only how $v$ and $u$ are connected in $G$ but also how they are partitioned by a graph partitioning scheme. 

\begin{remark}
Our model can be integrated into any existing GNN models as a plugin. Let $h_v^{gnn}$ be a node embedding from an existing GNN model (e.g., GIN~\cite{xu2018:powerful} and GCN~\cite{kipf2016:semi}. By combining $h_v^{gnn}$ with $\gamma_v^{(\ell+1)}$,  we obtain the representation for each vertex $v$ as 
\begin{align}
h_v=\textsc{Cmb}\left(h_v^{gnn}, \gamma_v^{(\ell+1)}\right).
\end{align}
\end{remark}
An illustration of the high-level workflow of GPNN is provided in Figure \ref{fig: partition diagram}.



\section{Expressivity Analysis}
\label{sec:theoretical-anslysis}

 We discuss the expressivity of GPNN from two aspects: (1) How does a partition colouring, determined by a chosen graph partitioning scheme, affect the expressivity of GPNN? (2) How do different types of interactions affect the expressivity of GPNN? Detailed proofs for theorems and lemmas presented in this section are included in the Appendix.

We use $\lambda$-GPNN$^{\delta}$ to refer to the variants of GPNN with a partition colouring $\lambda$, $d=1$ (i.e., $N_d(v)$ contains direct neighbours of $v$ and $v$ itself), and  $\delta\in\{\star, \diamond, \dagger\}$, where  $\star$, $\diamond$, and $\dagger$ denote that $(v,u)$ in \cref{eq:edge} considers intra-interactions  ($E^{\star}$), both inter-interactions and intra-interactions ($E^{\diamond}$), and all interactions ($E^{\dagger}$), respectively. We also introduce the notion of $\lambda$-equivalence to compare the expressivity of different graph partitioning schemes.

\begin{definition}[$\lambda$-Equivalence]
 Let $\lambda=(\lambda_V, \lambda_E)$ be a partition colouring. Two graphs $G$ and $H$ are \emph{$\lambda$-equivalent}, denoted as $G \equiv_{\lambda} H$, if they are indistinguishable by $\lambda$, i.e., 
 \begin{itemize}[itemsep=-0cm]
     \item for each $c\in C_V$, $|\{v\in V(G)|\lambda_V(v)=c\}|$ = $|\{v'\in V(H)|\lambda_V(v')=c\}|$;
     \item  for each $c\in C_E$, $|\{\{v,u\}\subseteq V(G)|\lambda_E(v,u)=c\}|$ = $|\{\{v',u'\}\subseteq V(H)|\lambda_E(v',u')=c\}|$.
 \end{itemize}

\end{definition}

Given two partition colourings $\lambda_1$ and $\lambda_2$, we say that \emph{$\lambda_1$ is at least as expressive as $\lambda_2$}, denoted as $\lambda_1  \sqsupseteq \lambda_2$,  if and only if for any two graphs $\{G,H\}\subseteq \mathcal{G}$, $G \equiv_{\lambda_2} H$ whenever $G \equiv_{\lambda_1} H$. Similarly, we use $\lambda_1\sqsupset \lambda_2$  and $\lambda_1\equiv\lambda_2$ to indicate that $\lambda_1$ is \emph{strictly more expressive than} or \emph{equally expressive as} $\lambda_2$, respectively. These notations can also be extended to $k$-WL and GPNNs. 

A partition colouring is \emph{trivial}, denoted as $\lambda_{\bot}$, if $f_{\Phi}(G)=\{G\}$ for any $G\in \mathcal{G}$, i.e., all vertices assigned the same colour. A partition colouring is \emph{complete}, denoted as $\lambda_{\top}$, if $f_{\Phi}(G)=f_{\Phi}(H)\Leftrightarrow G\simeq H$ for any $\{G,H\}\subseteq \mathcal{G}$, i.e., orbit colouring~\cite{mckay2014:practical}. The expressivity of a partition colouring $\lambda$ is lower bounded by $\lambda_{\bot}$ and upper bounded by $\lambda_{\top}$.

The following proposition describes how the expressivity of $\lambda$-GPNN$^{\delta}$ compares to 1-WL and 3-WL when $\lambda$ is trivial.
The expressive power of $\lambda_{\bot}$-GPNN$^{\delta}$ is at least as expressive as 1-WL but upper bounded by 3-WL.

\begin{restatable}[]{proposition}{thmlambdaone}\label{thm-triviallambda} 
   The following hold: (1) $\lambda_{\bot}$-GPNN$^{\star}\equiv 1\text{-WL}$; 
(2) $\lambda_{\bot}$-GPNN$^{\star}\sqsubseteq\lambda_{\bot}$-GPNN$^{\diamond}\sqsubseteq\lambda_{\bot}$-GPNN$^{\dagger}$; 
        (3) $\lambda_{\bot}$-GPNN$^{\dagger}\sqsubseteq 3\text{-WL}$. 
\end{restatable}




When considering the same interaction type, GPNN$^{\delta}$ maintains the expressivity order of their respective partition colourings, as stated in the theorem below.

\begin{restatable}[]{theorem}{thmtwolambdas}\label{thm-twolambdas} 
    Let $\lambda_1$ and $\lambda_2$ be two partition colourings with $\lambda_1 \sqsupseteq \lambda_2$. Then $\lambda_1$-GPNN$^{\delta}\sqsupseteq\lambda_2$-GPNN$^{\delta}$ for any $\delta\in\{\star, \diamond, \dagger\}$. 
\end{restatable}

The following proposition establishes expressivity bounds of $\lambda$-GPNN$^{\delta}$ in terms of the partition colouring $\lambda$ and $k$-WL. For clarity, we define $max(\lambda, \text{k-WL})=\text{k-WL}$ if $\lambda \sqsubseteq \text{k-WL}$, or $max(\lambda, \text{k-WL})=\lambda$ if $\lambda \sqsupseteq \text{k-WL}$. 

\begin{restatable}[]{theorem}{lemmodelpartition} 
\label{lem:model-partition}  Let $\lambda$ be a partition colouring satisfying $\lambda \sqsubseteq \text{k-WL}$ or $\lambda \sqsupseteq \text{k-WL}$ for some $k\in \mathbb{N}$. 
    Then $max(\lambda,\text{1-WL})\sqsubseteq \lambda\text{-GPNN}^{\delta}$ and $\lambda\text{-GPNN}^{\delta}\sqsubseteq max(\lambda,\text{3-WL})$  for any $\delta\in\{\star, \diamond, \dagger\}$.
\end{restatable}

We show that for a partition coloring $\lambda$ below 3-WL, the expressive power of $\lambda$-GPNN$^{\delta}$ is non decreasing when more interactions are captured into representations. However, when a partition coloring $\lambda$ is at least as expressive as 3-WL, the expressive power of $\lambda$-GPNN$^{\delta}$ remains unchanged. 

\begin{restatable}[]{lemma}{lemlessthreewl} \label{lem:less-3wl}
     When $\lambda \sqsubset$ 3-WL, $\lambda\text{-GPNN}^{\dagger}\sqsupseteq \lambda\text{-GPNN}^{\diamond}\sqsupseteq \lambda\text{-GPNN}^{\star}$.   When $\lambda\sqsupseteq$ 3-WL, $\lambda\text{-GPNN}^{\dagger}\equiv \lambda\text{-GPNN}^{\diamond}\equiv \lambda\text{-GPNN}^{\star}$.
\end{restatable}



Based on Lemma~\ref{lem:less-3wl}, the following theorem  compares the expressivity of different GPNN variants.

\begin{restatable}[]{theorem}{thmmain}\label{thm-main} 
For any partition colouring $\lambda$, $\lambda\text{-GPNN}^{\dagger} \sqsupseteq \lambda\text{-GPNN}^{\diamond}  \sqsupseteq\lambda\text{-GPNN}^{\star}$.
\end{restatable}

\begin{remark}
The expressivity of $\lambda\text{-GPNN}^{\delta}$ stems from two key sources: the capacity of a chosen graph partitioning scheme, as reflected by the partition colouring $\lambda$, and the ability to capture different types of structure interactions via $\delta\in\{\star, \diamond, \dagger\}$. The former sets the lower bound of $\lambda\text{-GPNN}^{\delta}$ in its ability to distinguish non-isomorphic graphs, ranging from a trivial colouring to a complete colouring. For the latter, $\lambda\text{-GPNN}^{\delta}$ addresses partition isomorphism and interaction isomorphism by considering various types of interactions. Inter-interactions enable $\lambda\text{-GPNN}^{\delta}$ to handle interactions within partitions, as required by partition isomorphism, while intra-interactions capture interactions across partitions, as considered by interaction isomorphism. The model capacity of GPNN$^\delta$ alone is bounded between 1-WL as a lower bound and 3-WL as an upper bound.
\end{remark}

\section{Complexity Analysis} \label{sec:complexity-analysis}
Generally, 1-WL and traditional GNNs that are upper-bounded by 1-WL, such as GIN~\cite{xu2018:powerful}, have a time complexity of $O(|E|)$. In contrast, 3-WL and existing GNNs with provable 3-WL expressive power are known to be computationally expensive, with a time complexity of at least $O(|V|^3)$~\cite{maron2019:provably}. As a result, while these GNNs can offer strong expressive power, they are often impractical for real-world applications. The time complexities of our GPNNs fall between those of 1-WL and 3-WL, offering a balanced trade-off between computational efficiency and representational power. Specifically, the time complexities of GPNN${^\star}$, GPNN${^\diamond}$, and GPNN${^\dagger}$ are $O(|E| + \dn |E^{\star}|)$, $O( |E| + \dn |E^{\diamond}|)$, and $O( |E| + \dn|V|^2)$, respectively, where $\dn = \textsc{Max}(\{ |N_d(v)| | v \in V(G) \})$.\cref{Tab:complexity_analysis} provides a detailed time complexity analysis of GPNNs, covering both vertex embedding and interaction embedding computations, and compares them with traditional GNNs and 3-WL.
 
\begin{table}[ht]
\centering
\resizebox{0.475\textwidth}{!}{ 
\begin{tabular}{r|cc|ccc|c} 
\toprule
 & GIN & GCN &  GPNN${^\star}$ & GPNN${^\diamond}$ & GPNN${^\dagger}$ & 3-WL \\ 
\midrule
\textbf{VE} & $O(|E|)$ & $O(|E|)$ &  $O(|E|)$ & $O(|E|)$ & $O(|E|)$ & - \\
\textbf{IE} & - & - &  $O(\dn |E^{\star}|)$ & $O(\dn |E^{\diamond}|)$ & $O(\dn |V|^2)$ & $O(|V|^3)$ \\
\bottomrule
\end{tabular}}
\caption{Time complexity analysis of different embedding computations, where \textbf{VE} refers to vertex embeddings and \textbf{IE} refers to interaction embeddings.}
\label{Tab:complexity_analysis}
\vspace{-0.25cm}
\end{table}
\vspace{0.25cm}
As shown in \cref{Tab:complexity_analysis}, GPNNs are more computationally efficient than 3-WL with a time complexity $O(|V|^3)$. This is due to two reasons. Firstly, real-world graphs are often sparse, and thus $|E^{\star}|$ and $|E^{\diamond}|$ are significantly smaller than $|V|^2$, i.e., $|E^{\star}| < |E^{\diamond}| <\!\!< |V^2|$. Secondly, GPNNs only consider local neighbourhoods for learning interaction embeddings, i.e., vertices within a $d$-hop neighborhood, while 3-WL adopts a global neighbourhood, where the neighbours of each edge involve every vertex in the entire graph.


It is worth noting that the graph partitioning scheme used by GPNNs should be computationally efficient, i.e.,  in linear or polynomial time in terms of the size of an input graph. This ensures that graph partitioning can be processed efficiently as a preprocessing step. In our experiments, we use a graph partitioning scheme with a time and space complexity of $O(|E|)$, which will be discussed in the next section.


\begin{table*}[t!]
\centering
\renewcommand\arraystretch{1.1}
\scalebox{0.8}{\begin{tabular}{c| c c c c c c |c c}
\specialrule{.1em}{.05em}{.05em} 
Methods & MUTAG & PTC-MR & PROTEINS & IMDB-B & BZR & COX2  & ogbg-molhiv & ogbg-moltox21\\ 
\toprule
{$^1$GIN} & {89.40 \sd{5.6}} & {64.60 \sd{7.0}} & {75.90 \sd{2.8}} & {75.10 \sd{5.1}}  & {85.60 \sd{2.0}} & {82.44 \sd{3.0}} & 75.58 \sd{1.4} & 74.91 \sd{0.5}\\ 

 {$^1$GraphSNN} & {91.24 \sd{2.5}} & {66.96 \sd{3.5}} & {76.51 \sd{2.5}}  & \textcolor{darkblue}{\textbf{76.93 \sd{3.3}}} & \textcolor{darkblue}{\textbf{88.69 \sd{3.2}}} & {82.86 
 \sd{3.1}} & {78.51 \sd{1.7}} & 75.45 \sd{1.1}\\ 
 {$^1$ESAN} & {91.00 \sd{7.1}} & \textcolor{darkblue}{\textbf{69.20 \sd{6.5}}} & {77.10 \sd{4.6}} & \textcolor{darkblue}{\textbf{77.10 \sd{3.0}}} & {N/A} & {N/A} & {76.43 \sd{2.1} } & {75.12 \sd{0.50}}\\
 {$^1$CIN} & \textcolor{darkblue}{\textbf{92.70 \sd{6.1}}} &{68.20 \sd{5.6}}& {77.00 \sd{4.3}} & {75.60 \sd{3.7}} & {N/A} & {N/A} & \textcolor{darkblue}{\textbf{80.94 \sd{0.6}}} & {N/A}\\ 
 {$^1$GIN-AK+} & {91.30 \sd{7.0}} & \textcolor{darkblue}{\textbf{68.20 \sd{5.6}}} &{77.10 \sd{5.7}} & {75.60 \sd{3.7}} & {N/A} & {N/A} & \textcolor{darkblue}{\textbf{79.61 \sd{1.1}}} & {N/A}\\ 
 {$^1$KP-GIN} & {92.20 \sd{6.5}} & {66.80 \sd{6.8}} & {75.80 \sd{4.6}} & {76.60 \sd{4.2}} & {N/A} & {N/A} & {N/A} & {N/A}\\ 
  {$^1$PPGN} & {90.55 \sd{8.7}} & {66.17 \sd{6.5}} & \textcolor{darkblue}{\textbf{ 77.20 \sd{4.7}}} & {73.00 \sd{5.8}} & {N/A} & {N/A} & {N/A} & {N/A}\\
 \bottomrule
$^2$GIN & 92.80 \sd{5.9} & 65.60 \sd{6.5} & 78.80 \sd{4.1} & 78.10 \sd{3.5} & 91.05 \sd{3.4} & 88.87 \sd{2.3} & 75.58 \sd{1.4} & 74.91 \sd{0.5}\\ 
$^2$GraphSNN & 94.70 \sd{1.9} & 70.58 \sd{3.1} & 78.42 \sd{2.7} & 78.51 \sd{2.8} & 91.12 \sd{3.0} & 86.28 \sd{3.3}  & {78.51 \sd{1.7}} & 75.45 \sd{1.1}\\ 
$^2$GIN-AK+ & 95.00 \sd{6.1} & 74.10 \sd{5.9} & 78.90 \sd{5.4} & 77.30 \sd{3.1} & {N/A} & {N/A} & {79.61 \sd{1.1}} & {N/A}\\ 
$^2$KP-GIN & 95.60 \sd{4.4} & 76.20 \sd{4.5} & 79.50 \sd{4.4} & \textbf{80.70 \sd{2.6}} & {N/A} & {N/A} & {N/A} & {N/A} \\ 
\bottomrule
{$^1$$\lambda_{\text{core-degree}}$-GPNN$^{\star}$} & {91.02 \sd{7.1}} & {66.20 \sd{11.2}} & \textcolor{darkblue}{\textbf{77.18 \sd{4.6}}} & {75.60 \sd{2.7}} & {88.60 \sd{4.6}} & \textcolor{darkblue}{\textbf{82.88 \sd{4.6}}} & 78.12 \sd{1.91} & \textcolor{darkblue}{\textbf{76.13 \sd{0.68}}} \\ 
 {$^1$$\lambda_{\text{core-degree}}$-GPNN$^{\diamond}$} & \textcolor{darkblue}{\textbf{92.60 \sd{4.8}}} & {65.95 \sd{8.5}} & {76.82 \sd{3.9}} & {74.40 \sd{2.4}} & \textcolor{darkblue}{\textbf{{89.12 \sd{2.3}}}} &  \textcolor{darkblue}{\textbf{83.09 \sd{3.1}}} & 77.63 \sd{1.80} & \textcolor{darkblue}{\textbf{75.89 \sd{0.30}}}\\ 
{$^2$$\lambda_{\text{core-degree}}$-GPNN$^{\star}$} & \textbf{97.89 \sd{2.6}} & \textbf{78.17 \sd{6.2}} & \textbf{81.40 \sd{3.5}} & \textbf{80.10 \sd{3.1}} & \textbf{94.05 \sd{2.6}} & \textbf{89.09 \sd{3.3}} & 78.12 \sd{1.91} & {76.13 \sd{0.68}} \\ 
{$^2$$\lambda_{\text{core-degree}}$-GPNN$^{\diamond}$} & \textbf{97.37 \sd{4.9}} & \textbf{79.61 \sd{7.3}} & \textbf{85.53 \sd{6.2}} & 78.10 \sd{3.1} & \textbf{91.84 \sd{1.6}} & \textbf{89.72 \sd{2.3}} & 77.63 \sd{1.80} & {75.89 \sd{0.30}}\\
\bottomrule
\end{tabular}}
\caption{Graph classification performance is reported as accuracy (\%) for TU datasets and ROC (\%) for OGB datasets. For the TU datasets, settings $1$ and $2$ correspond to the configurations used in \protect\citet{xu2018:powerful} and \protect\citet{feng2022:powerful}, respectively. For OGB datasets, we employ the setup introduced by \protect\citet{hu2020:open}. The best results for each setting are highlighted in \textcolor{darkblue}{\textbf{blue}} for setting $1$ and \textbf{black} for setting $2$. Baseline results are sourced from \protect\citet{feng2022:powerful}, \protect\citet{wijesinghe2022:new}, and \protect\citet{zhao2021:stars}.\vspace{-0.2cm}
}
\label{Tab:graph-classification-baselines}
\end{table*}

\section{Practical Choices of Partitioning Schemes}
One might ask how to choose a graph partitioning scheme. In practice, there are many design options available. However, two important criteria must be met:  \emph{permutation invariance} and \emph{computational efficiency}. To address these criteria, we first consider a graph partitioning scheme based on the k-core property \cite{malliaros2020core}, which is both permutation invariant and computationally efficient.

\begin{definition}[$k$-Core Property]\label{def:k-core}Let $\textsc{Deg}_S(v)$ denote the degree of vertex $v$ in a subgraph $S$. 
A subgraph $S$ has the \emph{$k$-core property} on a graph $G$  
if $S$ is the largest induced subgraph of $G$ satisfying:   $\forall v\in V(S)\hspace{0.2cm} \textsc{Deg}_S(v)\geq k$.
\end{definition}
Let $\varphi_j$ denote the $j$-core property and $\Phi_{core}=\{\phi_j\}_{j\in[0,k]}$ where $\phi_j=\varphi_j\wedge \neg\varphi_{j+1}$ satisfying the $j$-core property but not the $j\text{+}1$-core property. The graph partitioning scheme $f_{\Phi_{\text{core}}}$ corresponds to the shell decomposition~\cite{alvarez2005:large}, which decomposes a graph $G$ into a set of subgraphs, namely \emph{shells}. 
 Let $\phi\circ\psi$ denote the sequential composition of two properties $\phi$ and $\psi$. Then $\Phi_{\text{core}}$ can be further refined:  
\begin{itemize}[leftmargin=10pt,itemsep=-0cm]
    \item $\Phi_{\text{core-degree}}=\{\phi'_j\}_{j\in[0,2k]}$ where $\phi'_{2j-1}=\phi_j\circ \psi_j$, $\phi'_{2j}=\phi_j\circ \neg\psi_j$, and $\psi_j=\forall v\in V(S_j)\hspace{0.2cm} \textsc{Deg}_{S_j}(v)=j$.\looseness=-1
    \item $\Phi_{\text{core-onion}}=\{\phi'_{ji}\}_{j\in[0,k]}$ where $\phi'_{ji}=\phi_j\circ \psi^1_j \circ,\dots, $ $\circ \psi^i_j$, and $\psi_j^i$ refers to the $i$-th iteration to remove vertices with the lowest degree in the $j$-core subgraph.
\end{itemize}
Alternatively, we may consider graph partitioning schemes based on other properties such as vertex degrees. That is $\Phi_{\text{degree}} = \{\phi''_j\}_{j\in[0,k]}$ where $\phi''_j$ 
 is defined as the set of vertices with a degree of $j$, and $k$ is the maximum vertex degree in the graph. 
 When computational resources are sufficient, more computationally demanding schemes can be considered, such as $\Phi_{\text{triangle}}$ where the partitioning is based on the number of triangles in which each vertex participates.

A question that might arise is how a partition colouring $\lambda$ determined by such $f_{\Phi}$ relates to $k$-WL.
  The following theorem states that the expressive power of the partition colourings based on the $k$-core and degree properties are upper bounded by 1-WL.

\begin{restatable}[]{theorem}{thmkcore}\label{thm-kcore}  
 Let $\lambda$ be a partition colouring based on $f_{\Phi}$ where $\Phi\in \{ \Phi_{\text{core}},  \Phi_{\text{degree}}\}$. Then $G \equiv_{\lambda} H$ whenever $G\equiv_{\text{1-WL}} H$.\looseness=-1 
\end{restatable}
$\Phi_{\text{core-degree}}$ refines $\Phi_{\text{core}}$  by vertex degrees, while $\Phi_{\text{core-onion}}$  refines $\Phi_{\text{core}}$ based on degree correlations within each shell. \cref{thm-kcore} applies to both schemes since the refinements are based on vertex degrees. However, it does not hold for  $\Phi_{\text{triangle}}$, which can distinguish graphs that are indistinguishable by 1-WL. The proof is included in the Appendix.


\section{Experiments}\label{sec:experiments}

We conduct experiments on three widely used benchmark tasks: graph classification, graph regression, and node classification. 
Due to the limited space, we present experimental results for graph regression tasks in the Appendix.



\vspace{0.2cm}
\noindent\textbf{Datasets.}  We evaluate GPNNs on 15 benchmark datasets. For graph classification, we use six small-scale real-world datasets from TU Datasets \cite{morris2020:tudataset}
and three large-scale molecular datasets from Open Graph Benchmark (OGB) \cite{hu2020:open}. We also consider ZINC \cite{dwivedi2023:benchmarking} for graph regression, along with four widely used benchmark datasets for node classification \cite{sen2008:collective,craven2000:learning}. 


\vspace{0.15cm}
\noindent\textbf{Baselines.} 
  For graph classification, we choose GIN~\cite{xu2018:powerful} as the base model for GPNNs  and comparing them with GNNs beyond 1-WL: GIN, GraphSNN \cite{wijesinghe2022:new}, GIN-AK+ \cite{zhao2021:stars}, and KP-GIN \cite{feng2022:powerful}, CIN \cite{bodnar2021:weisfeiler}, PPGN \cite{maron2019:provably}, and ESAN \cite{bevilacqua2021:equivariant}. For node classification, we evaluate performance changes by incorporating our GPNN variants into standard GNNs, using GIN, GCN \cite{kipf2016:semi}, GAT \cite{velickovic2018:graph}, and GraphSAGE \cite{hamilton2017:inductive} as base models. 

  \vspace{0cm}
  Detailed information for dataset statistics, experimental setups, and hyper-parameters is provided in the Appendix.

\subsection{Results and Discussion}
\noindent\textbf{\emph{Exp--1. How do GPNNs perform for graph classification? }}\label{sec:graph-classification}  Table~\ref{Tab:graph-classification-baselines} reports the results. GPNN$^{\star}$ and GPNN$^{\diamond}$ consistently outperform or match the baseline models. Compared to GPNN$^{\diamond}$, GPNN$^{\star}$ handles fewer interactions but performs on par or exceeds the performance of GPNN$^{\diamond}$. This underscores that the presence of meaningful interactions are more paramount to learn graph structure than exhaustively exploring all interactions at the expense of computational efficiency. 
Unlike many existing models (e.g., \cite{bodnar2021:weisfeiler}), GPNNs do not rely on hand-crafted, domain-specific structural information, such as cycles, which are often correlated with molecular datasets like OGB. 


\vspace{0.2cm}
 \noindent\textbf{\emph{Exp--2. How do GPNNs perform for node classification? }} Table \ref{Tab:node_classification} presents the results, demonstrating that GPNNs consistently enhance the performance of all standards methods across benchmark datasets, including both homophilic and heterophilic ones. This enhancement is attributed to GPNNs' ability to incorporate structural interactions as additional features into representations, which standard GNNs often overlook. This allows GPNNs to capture nuanced relationships between vertices, leading to enhanced performance compared to methods that rely solely on inter-interactions between vertices.  

\begin{table}[ht!]
\hspace*{\oddsidemargin}
    \resizebox{0.52\textwidth}{!}{ 
    \begin{tabular}{c} 
        \begin{tabular}{ccccc} 
            \begin{subfigure}[b]{0.2\textwidth}
                \centering \;\;\;\;\;\;\;\; Degree
                \begin{tikzpicture}
                    \begin{axis}[
                        width=\textwidth,
                        ybar,
                        symbolic x coords={0, 1, 2, 3, 4, 5, 6},
                        xtick=data,
                        ymin=0, ymax=50,
                        xlabel={},
                        ylabel={moltoxcast},
                        bar width=4pt,
                        enlarge x limits=0.05,
                        every axis plot/.style={fill=blue},
                        xtick style={draw=none}
                    ]
                    \addplot coordinates {(0, 0) (1, 5) (2, 25) (3, 45) (4, 35) (5, 5) (6, 0)};
                    \end{axis}
                \end{tikzpicture}
            \end{subfigure} 
            \hspace{-0.8cm} 
            \begin{subfigure}[b]{0.2\textwidth}
                \centering \;\;\;\;\;\;\;\; Core \phantom{x}
                \begin{tikzpicture}
                    \begin{axis}[
                        width=\textwidth,
                        ybar,
                        symbolic x coords={0, 1, 2, 3, 4, 5, 6},
                        xtick=data,
                        ymin=0, ymax=70,
                        ytick={0, 30, 60},
                        xlabel={},
                        ylabel={},
                        bar width=4pt,
                        enlarge x limits=0.05,
                        every axis plot/.style={fill=blue},
                        xtick style={draw=none}
                    ]
                    \addplot coordinates {(0,0) (1, 5) (2, 50) (3, 60) (4, 0) (5, 0) (6,0)};
                    \end{axis}
                \end{tikzpicture}
            \end{subfigure}  
            \hspace{-1.0cm}
            \begin{subfigure}[b]{0.2\textwidth}
                \centering \;\;\;\; Core-Degree
                \begin{tikzpicture}
                    \begin{axis}[
                        width=\textwidth,
                        ybar,
                        symbolic x coords={0, 1, 2, 3, 4, 5, 6},
                        xtick=data,
                        ymin=0, ymax=50,
                        xlabel={},
                        ylabel={},
                        bar width=4pt,
                        enlarge x limits=0.05,
                        every axis plot/.style={fill=blue},
                        xtick style={draw=none}
                    ]
                    \addplot coordinates {(0,0) (1, 45) (2, 45) (3, 20) (4, 0) (5, 0) (6,0)};
                    \end{axis}
                \end{tikzpicture}
            \end{subfigure}
            \hspace{-1.0cm}
            \begin{subfigure}[b]{0.2\textwidth}
                \centering \;\;\;\; Core-Onion
                \begin{tikzpicture}
                    \begin{axis}[
                        width=\textwidth,
                        ybar,
                        symbolic x coords={0, 1, 2, 3, 4, 5, 6},
                        xtick=data,
                        ymin=0, ymax=25,
                        xlabel={},
                        ylabel={},
                        bar width=4pt,
                        enlarge x limits=0.05,
                        every axis plot/.style={fill=blue},
                        xtick style={draw=none}
                    ]
                    \addplot coordinates {(0,0) (1, 22) (2, 20) (3, 18) (4, 15) (5, 10) (6, 8)};
                    \end{axis}
                \end{tikzpicture}
            \end{subfigure}
            \hspace{-1.0cm}
            \begin{subfigure}[b]{0.2\textwidth}
                \centering \;\;\;\;\;\; Triangle
                \begin{tikzpicture}
                    \begin{axis}[
                        width=\textwidth,
                        ybar,
                        symbolic x coords={0, 1, 2, 3, 4, 5, 6},
                        xtick=data,
                        ymin=0, ymax=100,
                        ytick={0, 50, 100},
                        xlabel={},
                        ylabel={},
                        bar width=4pt,
                        enlarge x limits=0.05,
                        every axis plot/.style={fill=blue},
                        xtick style={draw=none}
                    ]
                    \addplot coordinates {(0,0) (1, 98) (2, 2) (3, 0) (4,0) (5,0) (6,0)};
                    \end{axis}
                \end{tikzpicture}
            \end{subfigure}
        \end{tabular}
        \\
        \begin{tabular}{ccc} 
            \begin{subfigure}[b]{0.2\textwidth}
                \centering 
                \begin{tikzpicture}
                    \begin{axis}[
                        width=\textwidth,
                        ybar,
                        symbolic x coords={0, 1, 2, 3, 4, 5, 6},
                        xtick=data,
                        ymin=0, ymax=50,
                        xlabel={},
                        ylabel={moltox21},
                        bar width=4pt,
                        enlarge x limits=0.05,
                        every axis plot/.style={fill=blue},
                        xtick style={draw=none}
                    ]
                    \addplot coordinates {(0,0) (1, 10) (2, 20) (3, 40) (4, 30) (5, 5) (6,0)};
                    \end{axis}
                \end{tikzpicture}
            \end{subfigure}
            \hspace{-0.8cm}
            \begin{subfigure}[b]{0.2\textwidth}
                \centering 
                \begin{tikzpicture}
                    \begin{axis}[
                        width=\textwidth,
                        ybar,
                        symbolic x coords={0, 1, 2, 3, 4, 5, 6},
                        xtick=data,
                        ymin=0, ymax=70,
                        ytick={0, 30, 60},
                        xlabel={},
                        ylabel={},
                        bar width=4pt,
                        enlarge x limits=0.05,
                        every axis plot/.style={fill=blue},
                        xtick style={draw=none}
                    ]
                    \addplot coordinates {(0,0) (1, 5) (2, 50) (3, 60) (4, 0) (5, 0) (6,0)};
                    \end{axis}
                \end{tikzpicture}
            \end{subfigure}
            \hspace{-1.0cm}
            \begin{subfigure}[b]{0.2\textwidth}
                \centering 
                \begin{tikzpicture}
                    \begin{axis}[
                        width=\textwidth,
                        ybar,
                        symbolic x coords={0, 1, 2, 3, 4, 5, 6},
                        xtick=data,
                        ymin=0, ymax=50,
                        xlabel={},
                        ylabel={},
                        bar width=4pt,
                        enlarge x limits=0.05,
                        every axis plot/.style={fill=blue},
                        xtick style={draw=none}
                    ]
                    \addplot coordinates {(0,0) (1, 45) (2, 45) (3, 20) (4, 0) (5, 0) (6,0)};
                    \end{axis}
                \end{tikzpicture}
            \end{subfigure}
            \hspace{-1.0cm}
            \begin{subfigure}[b]{0.2\textwidth}
                \centering 
                \begin{tikzpicture}
                    \begin{axis}[
                        width=\textwidth,
                        ybar,
                        symbolic x coords={0, 1, 2, 3, 4, 5, 6},
                        xtick=data,
                        ymin=0, ymax=25,
                        xlabel={},
                        ylabel={},
                        bar width=4pt,
                        enlarge x limits=0.05,
                        every axis plot/.style={fill=blue},
                        xtick style={draw=none}
                    ]
                    \addplot coordinates {(0,0) (1, 22) (2, 20) (3, 18) (4, 15) (5, 10) (6, 8)};
                    \end{axis}
                \end{tikzpicture}
            \end{subfigure}
            \hspace{-1.0cm}
            \begin{subfigure}[b]{0.2\textwidth}
                \centering 
                \begin{tikzpicture}
                    \begin{axis}[
                        width=\textwidth,
                        ybar,
                        symbolic x coords={0, 1, 2, 3, 4, 5, 6},
                        xtick=data,
                        ymin=0, ymax=100,
                        ytick={0, 50, 100},
                        xlabel={},
                        ylabel={},
                        bar width=4pt,
                        enlarge x limits=0.05,
                        every axis plot/.style={fill=blue},
                        xtick style={draw=none}
                    ]
                    \addplot coordinates {(0,0) (1, 97) (2, 3) (3, 0) (4,0) (5,0) (6,0)};
                    \end{axis}
                \end{tikzpicture}
            \end{subfigure}
        \end{tabular}
        \\
        \begin{tabular}{ccccc} 
            \begin{subfigure}[b]{0.2\textwidth}
                \centering
                \begin{tikzpicture}
                    \begin{axis}[
                        width=\textwidth,
                        ybar,
                        symbolic x coords={0, 1, 2, 3, 4, 5, 6},
                        xtick=data,
                        ymin=0, ymax=50,
                        xlabel={},
                        ylabel={molhiv},
                        bar width=4pt,
                        enlarge x limits=0.05,
                        every axis plot/.style={fill=blue},
                        xtick style={draw=none}
                    ]
                    \addplot coordinates {(0,0) (1, 5) (2, 15) (3, 45) (4, 35) (5, 5) (6,0)};
                    \end{axis}
                \end{tikzpicture}
            \end{subfigure}
            \hspace{-0.8cm}
            \begin{subfigure}[b]{0.2\textwidth}
                \centering 
                \begin{tikzpicture}
                    \begin{axis}[
                        width=\textwidth,
                        ybar,
                        symbolic x coords={0, 1, 2, 3, 4, 5, 6},
                        xtick=data,
                        ymin=0, ymax=70,
                        xlabel={},
                        ylabel={},
                        bar width=4pt,
                        enlarge x limits=0.05,
                        every axis plot/.style={fill=blue},
                        xtick style={draw=none}
                    ]
                    \addplot coordinates {(0,0) (1, 5) (2, 50) (3, 60) (4, 0) (5, 0) (6,0)};
                    \end{axis}
                \end{tikzpicture}
            \end{subfigure}
            \hspace{-1.0cm}
            \begin{subfigure}[b]{0.2\textwidth}
                \centering
                \begin{tikzpicture}
                    \begin{axis}[
                        width=\textwidth,
                        ybar,
                        symbolic x coords={0, 1, 2, 3, 4, 5, 6},
                        xtick=data,
                        ymin=0, ymax=50,
                        xlabel={},
                        ylabel={},
                        bar width=4pt,
                        enlarge x limits=0.05,
                        every axis plot/.style={fill=blue},
                        xtick style={draw=none}
                    ]
                    \addplot coordinates {(0,0) (1, 45) (2, 45) (3, 20) (4, 0) (5, 0) (6,0)};
                    \end{axis}
                \end{tikzpicture}
            \end{subfigure}
            \hspace{-1.0cm}
            \begin{subfigure}[b]{0.2\textwidth}
                \centering
                \begin{tikzpicture}
                    \begin{axis}[
                        width=\textwidth,
                        ybar,
                        symbolic x coords={1, 2, 3, 4, 5, 6},
                        xtick=data,
                        ymin=0, ymax=25,
                        xlabel={},
                        ylabel={},
                        bar width=4pt,
                        enlarge x limits=0.05,
                        every axis plot/.style={fill=blue},
                        xtick style={draw=none}
                    ]
                    \addplot coordinates {(1, 22) (2, 20) (3, 18) (4, 15) (5, 10) (6, 8)};
                    \end{axis}
                \end{tikzpicture}
            \end{subfigure}
            \hspace{-1.0cm}
            \begin{subfigure}[b]{0.2\textwidth}
                \centering
                \begin{tikzpicture}
                    \begin{axis}[
                        width=\textwidth,
                        ybar,
                        symbolic x coords={0, 1, 2, 3, 4, 5, 6},
                        xtick=data,
                        ymin=0, ymax=100,
                        ytick={0, 50, 100},
                        xlabel={},
                        ylabel={},
                        bar width=4pt,
                        enlarge x limits=0.05,
                        every axis plot/.style={fill=blue},
                        xtick style={draw=none}
                    ]
                    \addplot coordinates {(0,0) (1, 99) (2, 1) (3, 0) (4,0) (5,0) (6,0)};
                    \end{axis}
                \end{tikzpicture}
            \end{subfigure}
        \end{tabular}
    \end{tabular}
    }
    \captionof{figure}{Node distribution percentage (y-axis ) concerning partitions (x-axis) under different graph partitioning schemes for OGB Datasets.}
    \label{fig:node_distribution}\vspace{-0.2cm}
\end{table}

 \noindent\textbf{\emph{Exp--3: How does graph partitioning affect the learning of structural interactions?}} Table \ref{Tab:ablation2} shows the results for five partitioning schemes $\{\Phi_{\text{core}}, \Phi_{\text{core-degree}}, $ $\Phi_{\text{core-onion}}, \Phi_{\text{degree}}, \Phi_{\text{triangle}}\}$. The vertex distribution across partitions for each partitioning scheme is illustrated in \cref{fig:node_distribution}. GPNNs outperform GIN across all partitioning schemes and datasets. $\Phi_{\text{core-degree}}$ and $\Phi_{\text{degree}}$ excel because they incorporate a significant amount of structural interactions into representations and also generate a sufficient number of partitions without imposing excessive computational burden.
$\Phi_{\text{core-onion}}$ generates the most intra-interactions, but its high training complexity leads to subpar performance. $\Phi_\text{triangle}$
can distinguish graphs that 1-WL cannot, achieving strong performance even with fewer interactions.\looseness=-1

 \begin{table}[t!]
\resizebox{1\columnwidth}{!}{
\begin{tabular}{ccccc} 
\toprule
{Methods}& {Cora} & {CiteSeer}  & {Wisconsin} & {Texas}  \\ 
\toprule
{GIN} & {80.6 \sd{0.4}}  & {73.8 \sd{0.4}}  & {70.5 \sd{1.6}} & {61.6 \sd{1.1}}  \\
{{GPNN$^{\star}$$_{\text{GIN}}$}} & {82.0 \sd{1.4}}  & {\textbf{74.8 \sd{0.3}}}  & {71.3 \sd{1.7}}  & {\textbf{62.3 \sd{1.9}}}\\
{{GPNN$^{\diamond}$$_{\text{GIN}}$}} & {\textbf{82.3 \sd{0.9}}}  & {74.2 \sd{0.4}}  & {\textbf{71.4 \sd{1.5}}}  & {61.5 \sd{1.5}} \\
\midrule
{GCN} & {82.6 \sd{1.5}}  & {78.1 \sd{0.2}}  & {55.8 \sd{19.5}}  & {53.1 \sd{21.6}} \\
{{GPNN$^{\star}$$_{\text{GCN}}$}} & {\textbf{83.2 \sd{1.5}}}  & {\textbf{78.8 \sd{0.3}}}  & {\textbf{56.5 \sd{17.7}}}  & {51.8 \sd{23.0}} \\
{{GPNN$^{\diamond}$$_{\text{GCN}}$}} & {82.8 \sd{1.7}}  & {78.6 \sd{0.3}}  & {56.1 \sd{18.2}}  & {\textbf{57.7 \sd{22.5}}} \\
\midrule
{GAT} & {83.9 \sd{0.8}}  & {77.7 \sd{0.6}}  & {58.6 \sd{1.6}} & {62.0 \sd{2.2}} \\
{{GPNN$^{\star}$$_{\text{GAT}}$}} & {84.2 \sd{0.8}}  & {\textbf{78.1 \sd{0.6}}}  & {\textbf{60.3 \sd{2.5}}}  & {64.3 \sd{3.1}} \\
{{GPNN$^{\diamond}$$_{\text{GAT}}$}} & {\textbf{84.5 \sd{0.4}}}  & {77.8 \sd{0.8}}  & {59.9 \sd{2.5}}  & {\textbf{64.4 \sd{3.4}}} \\
\midrule
{GraphSAGE} & {84.5 \sd{0.3}}  & {77.9 \sd{0.5}}  & {87.0 \sd{1.7}}  & {81.6 \sd{2.3}} \\
{{GPNN$^{\star}$$_{\text{GraphSAGE}}$}} & {\textbf{84.9 \sd{0.3}}}  & {\textbf{78.2 \sd{0.6}}}  & {\textbf{87.4 \sd{1.2}}}  & {79.5 \sd{2.0}} \\
{{GPNN$^{\diamond}$$_{\text{GraphSAGE}}$}} & {84.7 \sd{0.5}}  & {78.1 \sd{0.5}}  & {87.1 \sd{1.6}} & {\textbf{82.0 \sd{1.3}}} \\
\bottomrule
\end{tabular}}
\caption{Node classification accuracy (\%).  The best results are highlighted in \textbf{black}. The partitioning scheme $\Phi_{core}$ is used in the experiment.}
\label{Tab:node_classification}
\vspace{0.3cm}
\centering
\resizebox{1\columnwidth}{!}{
\begin{tabular}{lccc} 
\toprule
{Methods}& \begin{tabular}{c}ogbg-\\moltoxcast\end{tabular} & \begin{tabular}{c}ogbg-\\moltox21\end{tabular} & \begin{tabular}{c}ogbg-\\molhiv\end{tabular}  \\
\toprule
{GIN\hspace{0.45cm}} & {63.41 \sd{0.7}} & {74.91 \sd{0.5}} & {75.58 \sd{1.4}} \\
\midrule
{$\lambda_{\text{core}}$-GPNN$^{\star}$} & \textbf{63.88 \sd{0.8}} & \textbf{75.34 \sd{0.4}} & \textbf{76.52 \sd{1.1}}\\
{$\lambda_{\text{core}}$-GPNN$^{\diamond}$} & {63.59 \sd{0.6}} & {75.28 \sd{0.5}} & {75.91 \sd{1.2}}\\
\midrule
{$\lambda_{\text{triangle}}$-GPNN$^{\star}$} & \textbf{63.99 \sd{0.4}} & \textbf{75.20 \sd{0.9}} & \textbf{76.42 \sd{1.2}}\\
{$\lambda_{\text{triangle}}$-GPNN$^{\diamond}$} & {63.41 \sd{0.7}} & {74.99 \sd{0.6}} & {75.92 \sd{1.2}}\\
\midrule
{$\lambda_{\text{core-onion}}$-GPNN$^{\star}$} & \textbf{62.91 \sd{0.8}} & {75.60 \sd{0.5}} & {76.57 \sd{1.1}}\\
{$\lambda_{\text{core-onion}}$-GPNN$^{\diamond}$} & {61.86 \sd{0.4}} & \textbf{76.07 \sd{0.5}} & \textbf{77.03 \sd{1.3}}\\
\midrule
{$\lambda_{\text{core-degree}}$-GPNN$^{\star}$} & \textbf{64.70 \sd{0.4}} & \textbf{76.13 \sd{0.7}} & \textbf{78.12 \sd{1.9}}\\
{$\lambda_{\text{core-degree}}$-GPNN$^{\diamond}$}  & {64.48 \sd{0.5}} & {75.98 \sd{0.4}} &  {77.70 \sd{2.2}}\\
\midrule
{$\lambda_{\text{degree}}$-GPNN$^{\star}$} & \textbf{65.28 \sd{0.5}} & {75.15 \sd{0.7}} & \textbf{78.98 \sd{1.5}}\\
{$\lambda_{\text{degree}}$-GPNN$^{\diamond}$} & {64.81 \sd{0.7}} & \textbf{75.89 \sd{0.3}} & {77.63 \sd{1.8}}\\
\bottomrule
\end{tabular}}
\caption{Comparison of GPNNs in graph classification ROC (\%) with five different graph partitioning schemes}
\label{Tab:ablation2}
\vspace{0.3cm}
\centering
\resizebox{1\columnwidth}{!}{
\begin{tabular}{ccccc} 
\toprule
{Methods}& {MUTAG} & {PTC\_MR} & {COX2} & {DHFR} \\ 
\toprule
{GIN\hspace{0.45cm}} & {92.80 \sd{5.9}} & {65.60 \sd{6.5}} & {88.87 \sd{2.3}} & {80.04 \sd{4.9}}\\
\midrule
{$\lambda_{\bot}$-GPNN$^{\star}$} & {94.74 \sd{4.7}} & {69.75 \sd{5.1}} & {87.37 \sd{4.1}} & {80.03 \sd{4.0}}\\
{$\lambda_{\bot}$-GPNN$^{\diamond}$} & \textbf{96.32 \sd{4.1}} & \textbf{71.77 \sd{4.3}} & {89.37 \sd{2.2}} & \textbf{80.30 \sd{3.0}}\\
{$\lambda_{\bot}$-GPNN$^{\dagger}$} & {95.76 \sd{3.2}} & {71.43 \sd{4.7}} & \textbf{89.45 \sd{2.1}} & {78.97 \sd{4.0}}\\
\midrule
{$\lambda_{\text{core-degree}}$-GPNN$^{\star}$} & {97.89 \sd{2.6}} & {78.17 \sd{6.2}} & {89.09 \sd{3.3}} & {82.01 \sd{2.3}}\\
{$\lambda_{\text{core-degree}}$-GPNN$^{\diamond}$} & \textbf{97.37 \sd{4.9}} & {79.61 \sd{7.3}} & \textbf{89.72 \sd{2.3}} & \textbf{82.15 \sd{2.9}}\\
{$\lambda_{\text{core-degree}}$-GPNN$^{\dagger}$} & {96.29 \sd{3.4}} & \textbf{80.05 \sd{7.1}} & {89.21 \sd{2.8}} & {80.16 \sd{2.6}}\\
\bottomrule
\end{tabular}}
\caption{Comparison of GPNN variants in graph classification accuracy (\%) with different interaction types.} 
\label{Tab:ablation}\vspace*{-0.4cm}
\end{table}

\vspace{0.2cm}
\noindent \textbf{\emph{Exp--4. How do different types of interactions affect performance?~}}
 We examine how trivial and non-trivial partitioning schemes, as well as various types of interactions, impact GPNNs' performance. For the partitioning scheme, we use $\Phi_{\text{core-degree}}$, which balances interaction exploration with computational efficiency. 
 Results are reported in Table \ref{Tab:ablation}. We can see that $\lambda_{\text{core-degree}}$-GPNN$^{\delta}$ consistently outperforms the corresponding $\lambda_{\bot}$-GPNN$^{\delta}$. 
All GPNNs show considerable improvements over GIN while $\lambda_{\bot}$-GPNN$^{\star}$ has the closest results, 
empirically validating \cref{thm-triviallambda}. Although GPNN$^{\dagger}$  theoretically possess higher expressive power, 
 its larger numbers of model parameters negatively affects model training and performance.

\section{Conclusion, Limitations and Future Work} \label{sec:conclusion}
We introduced the notion of permutation-invariant graph partitioning to explore complex interactions among structural components. Building on this, we developed a novel GNN architecture to integrate structural interactions into graph representation learning. Our theoretical analysis shows how these interactions can enhance the expressive power of GNN models, establishing strong connections to the 1-WL and 3-WL tests, as well as their complexities. Empirical results validate the effectiveness of the proposed model. 

 In this work, permutation-invariant partitioning schemes are not learned, and determining an effective scheme is largely application-dependent. From a theoretical perspective, however, it is an intriguing question how one might determine (or learn) optimal permutation-invariant partitioning schemes.  In future work, we plan to explore data-driven methods to automatically identify partitioning schemes that preserve key structural properties for downstream graph tasks, independent of vertex ordering.



\balance



\bibliography{references}
\bibliographystyle{icml2025}

\clearpage

\appendix
\section*{Appendix}

\bigskip

\section{Model Architecture} \label{sec:model-implementation}
 In this section, we describe the implementation details of our GPNN architecture. Let $\beta_v^{(0)} = \lambda_V(v)$, $\gamma_v^{(0)} = \lambda_V(v)$ and $\alpha_{vu}^{(0)} = \lambda_E(v, u)$.  Equations~\ref{eq:node} and \ref{eq:edge} for the vertex and interaction embeddings are implemented as follows:

\begin{equation}
\begin{split}
\beta_v^{(\ell+1)}=\textsc{Mlp}^{(\ell+1)}_{\theta} &\Biggl(\left(1 + \epsilon^{(\ell+1)}\right) \gamma_v^{(\ell)} \\& + \sum_{u\in N(v)}\gamma_u^{(\ell)}\Biggr); 
\end{split}
\end{equation}

\begin{equation}
\begin{split}
\alpha_{vu}^{(\ell+1)} = \textsc{Mlp}_{\psi}& \Biggl(
 \left(1 + \mu^{(\ell+1)}\right) \alpha_{vu}^{(\ell)} \\ &+ \sum_{w\in N_{d}(v)} \left(\alpha_{(vw)}^{(\ell)} + \alpha_{(vw)}^{(\ell)}\right) \Biggr).
\end{split}
\end{equation}

Here, $\beta_v^{(\ell+1)} \in \mathbb{R}^{f}$ is a vertex embedding; $\alpha_{vu}^{(\ell+1)} \in \mathbb{R}^{f}$ is an interaction embedding; $\gamma_v^{(\ell+1)} \in \mathbb{R}^{f}$ is an embedding that combines vertex embeddings and the corresponding interaction embeddings from the colored neighborhood of vertex $v$; $\mu^{(\ell+1)}$ and $\epsilon^{(\ell+1)}$ are learnable scalar parameters; $\textsc{Mlp}_{\theta}^{(\ell+1)}(.)$ and $\textsc{Mlp}_{\psi}^{(\ell+1)}(.)$ are multi-layer perceptron (MLP) functions, parameterized by $\theta$ and $\psi$, respectively.  
 \\\\
The embedding $\gamma^{(\ell+1)}_{v, j}$ for each \text{$j \in [1, k]$} is calculated as,

\begin{align}
    \gamma^{(\ell+1)}_{v, j} = \omega_j \left(\sum_{u \in N^j_{d}(v)} \left(\beta^{(\ell+1)}_u || \alpha^{(\ell+1)}_{vu} || p_j\right)  W_j\right). 
\end{align}

The combined embedding w.r.t. the whole neighborhood of vertex $v$ is defined as

\begin{equation}
 \gamma^{(\ell+1)}_v = \left(\gamma^{(\ell+1)}_{v, 1} + \dots + \gamma^{(\ell+1)}_{v, k} \right).  
\end{equation}

For any $j \in [1, k]$, $\gamma^{(\ell+1)}_{v, j}$ is the vertex embedding of $v$ learned at the $(\ell+1)$-th iteration concerning a neighboring vertex subset $N^j_{d}(v)$; $p_{j} \in \mathbb{R}^{1 \times k}$ is a one-hot vector in which values are all zero except one at the $j$-th position;  $W_{j} \in \mathbb{R}^{(2f+k) \times f}$  is a learnable linear transformation matrix; $\omega_j$ is a learnable scalar parameter. It is worth noting that $p_{j}$ is critical for preserving the injectivity of Equation \ref{eq:combine} in GPNN.
\\\\
When GPNN is used as a plug-in, $\gamma^{(\ell+1)}$ can be concatenated with the corresponding vertex embedding $h_v^{gnn}$ generated by a chosen GNN  as follows:
\begin{align}
    h_v= &\textsc{Concat} \left( h_v^{gnn}, \gamma^{(\ell+1)}\right)
\end{align}

Finally, we append a feature vector related to the number of connected components, provided by the graph partitioning scheme, to $h_v$ to enhance the representations.

\section{Proofs} \label{sec:proofs}

In the following we provide the proofs to the propositions, lemmas, and theorems presented in Section \ref{sec:theoretical-anslysis}. Without loss of generality, we assume that the functions $\textsc{Agg}(.), \textsc{Cmb}(.)$ and $\textsc{Upd}(.)$ used in GPNN are injective and there are also sufficiently many layers in GPNN.

\smallskip
\thmisomorphism*


\begin{proof}
Suppose that the graph partitioning scheme is $f_{\Phi}$. We first prove one direction of Statement (a), i.e., if $G \simeq G^{'}$, then $G\stackrel{II}{\simeq}G^{'}$. From the definition of graph isomorphism, we know that if $G \simeq G^{'}$, then there is a bijective mapping $f:V(G)\rightarrow V(G')$ such that for any $(u, v) \in E(G)$, $(f(u), f(v)) \in E(G^{'})$, and vice verse. By \cref{def:2}, we know that if $G\stackrel{PI}{\simeq}G'$, then there exists a bijective mapping $g:f_{\Phi}(G)\rightarrow f_{\Phi}(G')$ such that $g(S_i)\simeq S'_i$ for every $i \in [1, k]$. Thus, if $G \simeq G^{'}$, we know that the same bijective mapping from graph isomorphism $f:V(S_i)\rightarrow V(S_i^{'})$ satisfies that for any $(u, v) \in E(S_i)$, $(f(u), f(v)) \in E(S_i^{'})$ for any $i \in [1,k]$. Then, we consider the boundary subgraphs $B(G)=(\mathcal{V},\mathcal{E})$ and $B(G')=(\mathcal{V}', \mathcal{E}')$. By \cref{def:1}, it requires to show a bijective function $g: B(G) \rightarrow B(G')$ such that $(v,u)\in \mathcal{E}$ iff $(g(v), g(u)) \in \mathcal{E}'$. This can also be satisfied by $f$ where $f:\mathcal{V}\rightarrow \mathcal{V}'$ such that for any $(u, v) \in \mathcal{E}$, $(f(u), f(v)) \in \mathcal{E}'$. Therefore, such a bijective mapping $f$ of graph isomorphism satisfies both of the conditions of interaction-isomorphism. This direction is proven. 
However, the converse direction of Statement (a) does not hold. The pair of graphs shown in Figure \ref{fig:two-isomorphism}(a) is interaction-isomorphic, but not isomorphic.

For Statement (b), the first direction (i.e., if $G\stackrel{II}{\simeq}G^{'}$, then $G\stackrel{PI}{\simeq}G^{'}$) can easily follow from the definition of interaction-isomorphism. The converse do not hold, which can be proven by the pair of graphs shown in Figure \ref{fig:two-isomorphism}(b).

The proof is complete.
\end{proof}

Before proving the next theorem, we recall 2-FWL test~\cite{weisfeiler1968:reduction} which has been shown to be logically equivalent to 3-WL test \cite{maron2019:provably}. 2-FWL test initially provides a color $C^0(v, u)$ for each pair of vertices $v, u \in V(G)$, defined as follows. 
\begin{equation*}
  C^0(v, u) =
    \begin{cases}
    c_{self} &   u = v\\
      c_{interaction} &  (u, v) \in E(G)\\
      c_{non-interaction} &  (u, v) \notin E(G)\\
    \end{cases}       
\end{equation*}

Here, $c_{self}, c_{interaction}$ and $c_{non-interaction}$  are initial colours. Then the iterative coloring process of \text{2-FWL} is defined by:
\begin{equation}
    \begin{split}
        C^{(i+1)}(v, u) &= \sigma\left(C^{(i)}(v, u), \right. \\
        &\quad \left.\{\!\!\{(C^{(i)}(v, w), C^{(i)}(u, w)) \mid w \in V(G)\}\!\!\}\right).
    \end{split}
\end{equation}

Note that $\sigma(.)$ is an injective hashing function. The expressive power of 2-FWL is known to be strictly higher than 1-WL test and 2-WL test \cite{balcilar2021:breaking}.

\thmlambdaone*

\begin{proof}
    We first prove Statement (1) $\lambda_{\bot}$-GPNN$^{\star}\equiv 1\text{-WL}$. According to the model design, $\lambda_{\bot}$-GPNN$^{\star}$ starts with \cref{eq:node}. Since $\lambda_{\bot}$ is trivial, $E^{\star} = \emptyset$, meaning that all vertices are initially assigned the same colour. In \cref{eq:node}, a vertex embedding is updated based on its previous embedding and the embeddings of its direct neighbours. This implies that for, without any interation embedding, for any iteration of \cref{eq:node} has expressive power upper bounded by 1-WL,  as demonstrated in Theorem 3 of \citet{xu2018:powerful}. $\lambda_{\bot}$-GPNN$^{\star}$ does not account for any structural interactions. Therefore, \cref{eq:edge} does not contribute additional expressivity. 
    Given that \cref{eq:combine} is an injective combination of \cref{eq:node} and \cref{eq:edge} from the neighbours of each vertex and the vertex itself, the expressive power of the final embedding in any iteration is also upper-bounded by 1-WL. This completes the first part of the proof.

 We then prove Statement (2) $\lambda_{\bot}$-GPNN$^{\star}\sqsubseteq\lambda_{\bot}$-GPNN$^{\diamond}\sqsubseteq\lambda_{\bot}$-GPNN$^{\dagger}$. Given that \(E^{\star} \subseteq E^{\diamond} \subseteq E^{\dagger}\), it is trivial to see that $\lambda$-GPNN$^{\star} \sqsubseteq \lambda$-GPNN$^{\diamond} \sqsubseteq \lambda$-GPNN$^{\dagger}$. 
 

Finally, we prove Statement         (3) $\lambda_{\bot}$-GPNN$^{\dagger}\sqsubseteq 3\text{-WL}$.  When accounting for all of inter-interactions, intra-interactions and non-interactions in a graph, the expressive power of \cref{eq:edge} is upper-bounded by 2-FWL (i.e., equivalent to 3-WL in terms of their ability in distinguishing non-isomorphic graphs). Since \(\lambda\) is trivial, the initial vertex colours by $\lambda_{\bot}$ do not provide any additional power beyond that of 1-WL. Therefore, for any iteration, the final embeddings generated by any GPNN variant are upper bounded by 3-WL, meaning that \(\lambda\)-GPNN$^{\dagger} \sqsubseteq 3\text{-}WL$.
\end{proof}

\thmtwolambdas*

\begin{proof} By the definition of $\lambda$\text{-GPNN}$^{\delta}$, GPNN starts with a graph in which vertices are coloured by $\lambda$. Since $\lambda_1 \sqsupseteq \lambda_2$, we know that, for any two vertices $\{v,u\}\subseteq V(G)$, if $v$ and $u$ are assigned with the same colour by $\lambda_1$, then they must be assigned with the same colour by $\lambda_2$; the converse does not necessarily hold. Then, by the definition of GPNN in terms of Equations \ref{eq:node}, \ref{eq:edge}, and \ref{eq:combine}, we know that if two vertices $\{v,u\}\subseteq V(G)$ are assigned with different colours by $\lambda$, then they would always have different colours after applying $\lambda$\text{-GPNN}$^{\delta}$ with any number of layers. That is, $\lambda$\text{-GPNN}$^{\delta}$ always \emph{refines} a partitioning colouring $\lambda$. Since the same equations (Equations \ref{eq:node}, \ref{eq:edge}, and \ref{eq:combine}) are applied by $\lambda_1$\text{-GPNN}$^{\delta}$ and $\lambda_2$\text{-GPNN}$^{\delta}$ on vertices in $G$, which only differ in partitioning colourings, and also the functions $\textsc{Agg}(.), \textsc{Cmb}(.)$ and $\textsc{Upd}(.)$ used in GPNN are injective, any two vertices $v,u\in V(G)$ must have different colours by applying $\lambda_1$\text{-GPNN}$^{\delta}$ whenever they have different colours by applying $\lambda_2$\text{-GPNN}$^{\delta}$, but not vice versa. 
\end{proof}

\lemmodelpartition*

 \begin{proof} 

From \cref{thm-triviallambda}, we know that \(\lambda_{\bot}\)-GPNN$^{\star} \equiv \text{1-WL}$. By \cref{thm-twolambdas},   for any non-trivial \(\lambda\), \(\lambda\)-GPNN$^{\star}$ is at least as expressive as \(\lambda_{\bot}\)-GPNN$^{\star}$. Hence, we can deduce that \(\lambda\)-GPNN$^{\star}$ is at least lower bounded by 1-WL.  When \(\lambda\) is strictly more expressive than 1-WL, any iteration of GPNN will be strictly more expressive than the corresponding 1-WL iteration as \cref{eq:node} and \cref{eq:combine} together can only further refine initial vertex colours. In this case, \(\lambda\)-GPNN$^{\star}$ is lower bounded by \(\lambda\). Therefore, we can conclude that, for any \(\lambda\), $\max(\lambda, \text{1-WL}) \sqsubseteq \lambda\text{-GPNN}^{\star}$. Given that $E^{\star} \subseteq E^{\diamond} \subseteq E^{\dagger}$, we can deduce that  $\lambda\text{-GPNN}^{\star} \sqsubseteq \lambda\text{-GPNN}^{\diamond} \sqsubseteq \lambda\text{-GPNN}^{\dagger}$. Therefore, this relationship also holds for the other two GPNN variants, as they are at least as expressive as the corresponding $\lambda$-GPNN$^{\star}$.

From \cref{thm-triviallambda}, we also know that \(\lambda_{\bot}\)-GPNN$^{\delta}$ is upper bounded by 3-WL. Additionally, if \(\lambda \sqsubseteq \text{3-WL}\), then for any non-trivial \(\lambda\), \(\lambda\)-GPNN$^{\delta}$ is also upper bounded by 3-WL. This is because initial vertex coloring is upper bounded by 3-WL, and any interaction learning in GPNNs is similarly constrained by 3-WL, leading to final embeddings that are upper bounded by 3-WL. When \(\lambda\) is strictly more powerful than 3-WL, GPNN is as expressive as \(\lambda\), since no iteration of GPNN can further refine the color beyond \(\lambda\). Therefore, for any \(\lambda\), we have $\lambda\text{-GPNN}^{\dagger} \sqsubseteq \max(\lambda, \text{3-WL})$.
\end{proof}

\lemlessthreewl*

\begin{proof}

First, we consider the case where $\lambda \sqsubset \text{3-WL}$. Since $E^\star \subseteq E^\diamond \subseteq E^\dagger$, for any $\lambda$, we can derive the following: $\lambda$\text{-GPNN$^{\star}$} $\sqsubseteq$ $\lambda$\text{-GPNN$^{\diamond}$} $\sqsubseteq$ $\lambda$\text{-GPNN$^{\dagger}$}.


Next, we consider the case where $\lambda \sqsupseteq \text{3-WL}$. In this case, by $max(\lambda,\text{1-WL})\sqsubseteq \lambda\text{-GPNN}^{\delta}$, as stated in \cref{lem:model-partition}, we know that 
the expressive power of $\lambda\text{-GPNN}^{\delta}$ is lower-bounded by $\lambda$. Similarly, by $\lambda\text{-GPNN}^{\delta}\sqsubseteq max(\lambda,\text{3-WL})$ from \cref{lem:model-partition}, we know that 
the expressive power of $\lambda\text{-GPNN}^{\delta}$ is also upper-bounded by $\lambda$. Therefore, all GPNN variants ($\lambda\text{-GPNN}^{\delta}$) will have the same expressive power, determined by the initial partition coloring 
$\lambda$, regardless of the interaction types $\delta\in\{\star, \diamond, \dagger\}$ considered. Thus, $\lambda\text{-GPNN}^{\dagger}\equiv \lambda\text{-GPNN}^{\diamond}\equiv \lambda\text{-GPNN}^{\star}$ holds.
 \end{proof}

 \thmmain*

 \begin{proof}  
Given that \(\lambda\) determines the initial vertex colouring in GPNNs, any pair of graphs distinguished by \(\lambda\) will also be distinguished by any GPNN variant. Additionally, since \(E^{\star} \subseteq E^{\diamond} \subseteq E^{\dagger}\), it follows that for any \(\lambda\), $\text{GPNN}^{\star} \sqsubseteq \text{GPNN}^{\diamond} \sqsubseteq \text{GPNN}^{\dagger}$.This completes the proof.
 \end{proof}

\begin{table*}[ht]
\centering
\resizebox{1.5\columnwidth}{!}{
\begin{tabular}{lccccc} 
\toprule
{Datasets} & {Task Type} & \# Graphs & Avg \#Nodes & Avg \#Edges & \# Classes \\ 
\toprule
{MUTAG} & {Graph Classification} & {188} & {17.93} & {19.79} & {2} \\
{PTC\_MR} & {Graph Classification} & {344} & {14.29} & { 14.69} & {2} \\
{BZR} & {Graph Classification} & {405} & {35.75} & { 38.36} & {2} \\
{COX2} & {Graph Classification} & {467} & { 41.22} & {43.45} & {2} \\
{DHFR} & {Graph Classification} & {756} & {42.43} & {44.54} & {2} \\
{IMDB-B} & {Graph Classification} & {1,000} & {19.77} & {96.53} & {2} \\
{PROTEINS} & {Graph Classification} & {1,113} & {39.06} & { 72.82} & {2} \\
\midrule
ogbg-moltox21 & {Graph Classification} & {7,831} & {18.6} & {19.3} & {2} \\
ogbg-moltoxcast & {Graph Classification} & {8,576} & {18.8} & {19.3} & {2} \\
ogbg-molhiv & {Graph Classification} & {41,127} & { 25.5} & { 27.5} & {2} \\
\midrule
{ZINC} & {Graph Regression} & {12,000} & {23.1} & {49.8} & {1} \\
\bottomrule
\end{tabular}}
\caption{Dataset statistics for graph classification and regression}
\label{tbl:statistics}
\end{table*}
 
 \thmkcore*
\begin{proof}
We first prove the theorem for $\Phi_{core}$. Let $\lambda_{1\text{-}WL}$ and $\lambda_{core}$ be two partition colorings generated by 1-WL and core decomposition based on the k-core property (Definition~\ref{def:k-core}), respectively. 1-WL is a coloring algorithm that iteratively captures the neighborhood degree sequence of each vertex in a graph \cite{smith2019:neighbourhood}. If $\lambda_{1-WL}(u) = \lambda_{1-WL}(v)$, then the vertices $u$ and $v$ must have identical neighborhood degree sequence. Core decomposition is an iterative peeling algorithm that, given an input graph $G$, the peeling process starts from vertices of the lowest degree in the graph $G$. Assume that there are two vertices $u$ and $v$ s.t. $\lambda_{1-WL}(u) = \lambda_{1-WL}(v)$ and $\lambda_{core}(u) \neq \lambda_{core}(v)$. Since $\lambda_{core}(u) \neq \lambda_{core}(v)$, $u$ and $v$ must appear in different peeling iterations; otherwise, they would have the same color from the core decomposition. Let $P_u$ and $P_v$ be two sets containing all vertices removed in the peeling iterations before $u$ and $v$ are removed by the core decomposition algorithm, respectively. Since $\lambda_{core}(u)$ and $\lambda_{core}(v)$ are different, K-shell numbers \cite{alvarez2005:large} of $u$ and $v$ must be distinct. Without loss of generality, we assume that the K-shell value of $v$ is larger than the K-shell value of $u$. Then, $P_u$ must be a proper subset of $P_v$ (i.e. $P_u \subset P_v$). In other words, $P_u$ cannot be the same as $P_v$. Thus, $P_u$ and $P_v$ should have different degree sequences. Since, $u$ and $v$ have different degree sequences around them, their coloring received from 1-WL should be different. This contradicts the above assumption. Therefore, $\lambda_{1-WL}(u) = \lambda_{1-WL}(v) \implies \lambda_{core}(u) = \lambda_{core}(v)$. This also implies that if the colors of any two nodes in two graphs $G$ and $H$ are the same by 1-WL,  then their node colors by core decomposition must also be the same. Therefore, under an injective graph readout function, if $G\equiv_{\text{1-WL}} H$, then $G\equiv_{\text{core}} H$.

Then, we prove the theorem for $\Phi_{\text{degree}}$. Let $\lambda_{1\text{-}WL}$ and $\lambda_{degree}$ be two partition colourings generated by 1-WL and degrees of vertices, respectively. Since the 1\textsuperscript{st} iteration of 1-WL captures the 1-hop neighborhood degree sequence of each vertex, the colouring produced by the 1\textsuperscript{st} iteration of 1-WL must be equivalent to the colouring of $\lambda_{degree}$. Further iterations of 1-WL would only refine the colours of vertices.
This implies that if the colours of any two vertices in two graphs $G$ and $H$ are the same by 1-WL,  then their vertex colours by $\lambda_{degree}$ must also be the same, i.e. $\forall u, v \in V(G), \lambda_{1-WL}(u) = \lambda_{1-WL}(v) \implies \lambda_{degree}(u) = \lambda_{degree}(v)$, although the converse does not generally hold. Therefore, under an injective graph readout function, if $G\equiv_{\text{1-WL}} H$, then $G\equiv_{\text{degree}} H$. 
\end{proof}

\section{Supplementary Experimental Details} \label{sec:experimental-design}

In this section, we describe the statistics of the benchmark datasets, model hyper-parameters, and computational resources related to our experiments. 

\subsection{Benchmark Dataset Statistics} \label{sec:benchmark-datasets}

Table~\ref{tbl:statistics} and Table~\ref{tbl:statistics_vertex}, provide a summary of statistics for datasets that we have considered in all our experiments. Table~\ref{tbl:statistics} consists of the dataset statistics for graph classification and graph regression, whereas Table~\ref{tbl:statistics_vertex} contains the dataset statistics for node classification.

\begin{table}[H]
\centering
\resizebox{0.98\columnwidth}{!}{
\begin{tabular}{lcccc} 
\toprule
{Datasets}  & \# Nodes & \# Edges & \# Features  & \# Classes \\ 
\toprule
{Cora}  & {2708} & {5278} & {1433} & {7} \\
{CiteSeer}  & {3327} & {4552} & {3703} & {6} \\
{Wisconsin}  & {251} & {466} & {1703} & {5} \\
{Texas}  & {183} & {279} & {1703} & {5} \\
\bottomrule
\end{tabular}}
\caption{Dataset statistics for node classification }
\label{tbl:statistics_vertex}
\end{table}

\subsection{Experimental Setups} In graph classification, we adopt the setup outlined by  \protect\citet{hu2020:open} for the OGB datasets. For TU datasets, we present results under two distinct configurations: Setting 1 as described by \protect\citet{xu2018:powerful} and Setting 2 by \protect\citet{feng2022:powerful}.
In the first setting, we report the average test accuracy across 10 folds and the standard deviation for the epoch with the highest mean accuracy. In the second setting, we provide the mean and standard deviation of the best accuracy across all test folds. 
For node classification, we follow the setup suggested in \protect\citet{pei2020:geom}, where vertices in the datasets are randomly split into 60\%, 20\%, and 20\% for training, validation, and testing, respectively.

\begin{table*}[h]
\centering
\resizebox{1.8\columnwidth}{!}{
\begin{tabular}{clcccc} 
\toprule
\multirow{2}{*}{Dataset} & {Partitioning}& \multirow{2}{*}{\# Partitions} & {Avg } & {Avg} &  \multirow{2}{*}{\# Intra-Edges \slash\;\# Inter-Edges} \\
 & {Scheme}&  & \# Intra-Edges &  \# Inter-Edges &   \\ 
\toprule
\multirow{5}{*}{ogbg-moltoxcast} & {$\Phi_{\text{core}}$} & {3} & {5.08}  & {33.44} & {13 \slash\; 87}\\ 
& {$\Phi_{\text{core-degree}}$} & {5} & {13.05}  & {25.47} & {34 \slash\; 66}\\ 
& {$\Phi_{\text{core-onion}}$} & {42} & {24.53}  & {13.99} & {64 \slash\; 36}\\ 
& {$\Phi_{\text{degree}}$} & {7} & {24.56}  & {13.96} & {64 \slash\; 36}\\ 
& {$\Phi_{\text{triangle}}$} & {2} & {0.09}  & {38.43}& {\hspace{0.2cm}1 \slash\; 99}\\ 
\midrule
\multirow{5}{*}{ogbg-moltox21} & {$\Phi_{\text{core}}$} & {3} & {5.12}  & {33.47} & {13 \slash\; 87}\\ 
& {$\Phi_{\text{core-degree}}$} & {5} & {13.18}  & {25.41} & {34 \slash\; 66}\\ 
& {$\Phi_{\text{core-onion}}$} & {42} & {24.54}  & {14.05} & {64 \slash\; 36}\\ 
& {$\Phi_{\text{degree}}$} & {7} & {24.50}  & {14.08} & {63 \slash\; 37}\\ 
& {$\Phi_{\text{triangle}}$} & {2} & {0.09}  & {38.50}& {\hspace{0.2cm}1 \slash\; 99}\\
\midrule
\multirow{5}{*}{ogbg-molhiv} & {$\Phi_{\text{core}}$} & {4} & {7.09}  & {47.84} & {13 \slash\; 87}\\ 
& {$\Phi_{\text{core-degree}}$} & {8} & {23.30}  & {31.64} & {42 \slash\; 58}\\ 
& {$\Phi_{\text{core-onion}}$} & {88} & {33.53}  & {21.41} & {61 \slash\; 39}\\ 
& {$\Phi_{\text{degree}}$} & {11} & {33.52}  & {21.42} & {61 \slash\; 39}\\ 
& {$\Phi_{\text{triangle}}$} & {11} & {0.13}  & {54.80}& {\hspace{0.2cm}1 \slash\; 99}\\
\bottomrule
\end{tabular}}
\caption{Statistics for five graph partitioning schemes used in the experiments} 
\label{Tab:ablation3-stats}
\end{table*}

\subsection{Hyper-parameters} 
\label{sec:hyper-parameters} 

In our experiments, the hyper-parameters of GPNN are searched in the following ranges:  the number of layers is within $ \{1 ,2, 3, 4, 5\}$, $d$ $\in \{1 ,2, 3, 4, 5\}$, dropout $\in \{0.0, 0.1 ,0.2, 0.5, 0.6\}$, learning rate $\in \{0.009, 0.01\}$, batch size $\in \{32, 64, 128\}$, hidden units $\in \{32, 64, 100\}$, and the number of epochs $\in \{100, 200, 300, 500\}$. We use the Adam algorithm \cite{kingma2014:adam} as our optimizer. For ZINC and OGB datasets, the initial learning rate decays with a factor of 0.5 after every 10 epochs. We do not use any learning rate decay technique for TU datasets.

\subsection{Graph Regression Tasks}
We provide additional experiments related to graph regression and discuss the impact of different graph partition schemes on the performance of GPNNs. In this experiments, we follow the experimental setup described by \citet{dwivedi2023:benchmarking}.
We consider the ZINC dataset \cite{dwivedi2023:benchmarking} and use a 12K subset of the ZINC (250K) dataset \cite{irwin2012:zinc} with and without edge features. we employ GIN \cite{xu2018:powerful}, GCN \cite{kipf2016:semi}, PPGN \cite{maron2019:provably}, DGN \cite{beaini2021:directional}, Deep LRP \cite{chen2020:can}, and CIN \cite{bodnar2021:weisfeiler} as baselines. The results are summarized in Table \ref{Tab:regression}.

\begin{table}[H]
\centering
\resizebox{0.98\columnwidth}{!}{
\begin{tabular}{ccc} 
\toprule
\multirow{2}{*}{Methods} & \multicolumn{2}{c}{ZINC} \\ 
\cmidrule(lr){2-3} & {\small Without Edge Features} & {\small With Edge Features}  \\ 
\midrule
{GCN} & {0.459 \sd{0.006}}   & {0.321 \sd{0.009}}   \\
{PPGN} & {0.407 \sd{0.028}} & {-}  \\
{GIN} & {0.407 \sd{0.028}} & {0.163 \sd{0.004}}   \\
{DGN} & {0.219 \sd{0.010}}  & {0.168 \sd{0.003}}  \\
\midrule
{Deep LRP$^{\#}$} & {0.223 \sd{0.008}} & {-} \\
{CIN$^{\#}$} & \textbf{{0.115 \sd{0.003}}} & \textbf{{0.079 \sd{0.006}}}\\
\midrule

{$\lambda_{\text{core-degree}}$-GPNN$^{\star}$} & \textcolor{darkblue}{\textbf{0.214 \sd{0.018}}} & {0.148 \sd{0.015}} \\
{$\lambda_{\text{core-degree}}$-GPNN$^{\diamond}$} & {0.222 \sd{0.021}} & \textcolor{darkblue}{\textbf{{0.141 \sd{0.011}}}} \\
\bottomrule
\end{tabular}}
\caption{Graph regression MAE for ZINC 12K dataset.The best results are highlighted in \textbf{black} and the second best results are highlighted in \textcolor{darkblue}{\textbf{blue}}. For the baseline results, we refer to \protect\cite{bodnar2021:weisfeiler}. The  methods marked by \#
are based on pre-selected domain-specific structural information. } 
\label{Tab:regression}
\end{table}

From the results in  \cref{Tab:regression}, we can see that both GPNN variants show competitive performance in the regression task. Different from models like CIN,  GPNN does not process any hand-crafted domain-specific structural information that is known to be highly correlated to molecular prediction tasks. Therefore, our model performance is less than such models.

\subsection{Computational Resources} \label{sec:computation-resources}

 All our experiments are conducted on a Linux server with an Intel Xeon W-2175 2.50GHz processor alongside 28 cores, NVIDIA RTX A6000 GPU, and 512GB of main memory.

\end{document}